\documentclass[10pt,letterpaper]{article}
\usepackage[margin=1in]{geometry}

\usepackage{microtype}
\usepackage{graphicx}
\usepackage{subcaption}
\usepackage{booktabs} % for professional tables
\usepackage{colortbl}
\usepackage{multirow}
\usepackage{natbib}

\usepackage{hyperref}
\hypersetup{
  pdftitle={From Text to Forecasts: Bridging Modality Gap with Temporal Evolution Semantic Space},
  pdfauthor={Lehui Li, Yuyao Wang, Jisheng Yan, Wei Zhang, Jinliang Deng, Haoliang Sun, Zhongyi Han, Yongshun Gong},
  pdfsubject={Preprint},
  pdfkeywords={Time-Series Forecasting, Multimodal Learning, Temporal Evolution Semantic Space}
}

\usepackage{tikz}

\newcommand{\best}[1]{\textcolor{red}{\textbf{#1}}}
\newcommand{\second}[1]{\textcolor{blue}{\textbf{#1}}}
\newcommand{\method}{\textsc{TESS}}

\newcommand{\cV}{\mathcal V}
\newcommand{\E}{\mathbb E}
\newcommand{\e}{\varepsilon} % (optional)

\newcommand{\bz}{\mathbf z}
\newcommand{\err}{\mathrm{err}}
\newcommand{\true}{\mathrm{true}}
\newcommand{\timev}{\mathrm{time}}

\newcommand{\circled}[2]{%
  \tikz[baseline=(char.base)]{
    \node[shape=circle,draw=#1,fill=#1!20,inner sep=1.5pt,minimum size=1.2em] (char) {\textbf{\small #2}};
  }%
}
\newcommand{\cOne}{\circled{blue}{1}}
\newcommand{\cTwo}{\circled{orange}{2}}
\newcommand{\cI}{\circled{teal}{i}}
\newcommand{\cII}{\circled{purple}{ii}}
\newcommand{\cIII}{\circled{red}{iii}}

\newcommand{\glyphStrongRise}{\tikz[baseline=-0.3ex]{\draw[->,thick,red!70!black] (0,0) -- (0,0.5); \draw[->,thick,red!70!black] (0.12,0) -- (0.12,0.5);}}
\newcommand{\glyphMildRise}{\tikz[baseline=-0.3ex]{\draw[->,thick,orange] (0,0) -- (0,0.5);}}
\newcommand{\glyphStable}{\tikz[baseline=0.1ex]{\draw[-,thick,gray] (0,0) -- (0.4,0);}}
\newcommand{\glyphMildDrop}{\tikz[baseline=-0.3ex]{\draw[->,thick,cyan!70!blue] (0,0.5) -- (0,0);}}
\newcommand{\glyphStrongDrop}{\tikz[baseline=-0.3ex]{\draw[->,thick,blue!70!black] (0,0.5) -- (0,0); \draw[->,thick,blue!70!black] (0.12,0.5) -- (0.12,0);}}

\newcommand{\glyphAscend}{\tikz[baseline=0.1ex]{%
  \draw[line width=1.2pt,red!70!black] (0,0) .. controls (0.15,0.1) and (0.25,0.4) .. (0.45,0.5);
  \fill[red!70!black] (0.42,0.47) -- (0.5,0.52) -- (0.42,0.42) -- cycle;}}
\newcommand{\glyphDescend}{\tikz[baseline=0.1ex]{%
  \draw[line width=1.2pt,blue!70!black] (0,0.5) .. controls (0.15,0.4) and (0.25,0.1) .. (0.45,0);
  \fill[blue!70!black] (0.42,0.08) -- (0.5,-0.02) -- (0.42,0.03) -- cycle;}}
\newcommand{\glyphPeak}{\tikz[baseline=0.05ex]{%
  \draw[line width=1.2pt,orange!85!black] (0,0.08) .. controls (0.1,0.1) and (0.15,0.45) .. (0.23,0.48) 
    .. controls (0.31,0.45) and (0.36,0.1) .. (0.46,0.08);
  \fill[orange!85!black] (0.23,0.52) circle (0.04);}}
\newcommand{\glyphTrough}{\tikz[baseline=0.05ex]{%
  \draw[line width=1.2pt,teal!80!black] (0,0.42) .. controls (0.1,0.4) and (0.15,0.05) .. (0.23,0.02) 
    .. controls (0.31,0.05) and (0.36,0.4) .. (0.46,0.42);
  \fill[teal!80!black] (0.23,-0.02) circle (0.04);}}
\newcommand{\glyphOscillate}{\tikz[baseline=0.15ex]{%
  \draw[line width=1.2pt,purple!80!black] (0,0.22) sin (0.08,0.42) cos (0.16,0.22) sin (0.24,0.02) cos (0.32,0.22) sin (0.4,0.42) cos (0.48,0.22);}}

\newcommand{\glyphSurge}{\tikz[baseline=0.12ex]{%
  \draw[line width=1pt,red!70!black] (0,0.2) -- (0.07,0.5) -- (0.14,0) -- (0.21,0.48) -- (0.28,0.05) -- (0.35,0.45) -- (0.42,0.1);}}
\newcommand{\glyphVolRise}{\tikz[baseline=0.12ex]{%
  \draw[line width=1pt,orange!85!black] (0,0.22) -- (0.08,0.32) -- (0.16,0.12) -- (0.24,0.38) -- (0.32,0.08) -- (0.4,0.42);}}
\newcommand{\glyphVolStable}{\tikz[baseline=0.15ex]{%
  \draw[line width=1pt,gray!70!black] (0,0.22) -- (0.1,0.28) -- (0.2,0.18) -- (0.3,0.26) -- (0.4,0.2);}}
\newcommand{\glyphVolFall}{\tikz[baseline=0.12ex]{%
  \draw[line width=1pt,cyan!70!blue] (0,0.1) -- (0.08,0.42) -- (0.16,0.05) -- (0.24,0.35) -- (0.32,0.18) -- (0.4,0.28);}}
\newcommand{\glyphCalm}{\tikz[baseline=0.18ex]{%
  \draw[line width=1pt,blue!60!black] (0,0.22) -- (0.1,0.24) -- (0.2,0.21) -- (0.3,0.23) -- (0.4,0.22);}}

\newcommand{\glyphEarlyFade}{\tikz[baseline=0ex]{\fill[red!70!black] (0,0) rectangle (0.12,0.45); \fill[gray!50] (0.14,0) rectangle (0.26,0.2); \fill[gray!30] (0.28,0) rectangle (0.4,0.1);}}
\newcommand{\glyphEarlyPersist}{\tikz[baseline=0ex]{\fill[red!70!black] (0,0) rectangle (0.12,0.45); \fill[red!60!black] (0.14,0) rectangle (0.26,0.4); \fill[red!50!black] (0.28,0) rectangle (0.4,0.35);}}
\newcommand{\glyphMidFade}{\tikz[baseline=0ex]{\fill[gray!30] (0,0) rectangle (0.12,0.15); \fill[orange] (0.14,0) rectangle (0.26,0.45); \fill[gray!30] (0.28,0) rectangle (0.4,0.15);}}
\newcommand{\glyphMidPersist}{\tikz[baseline=0ex]{\fill[gray!30] (0,0) rectangle (0.12,0.15); \fill[orange] (0.14,0) rectangle (0.26,0.45); \fill[orange!80] (0.28,0) rectangle (0.4,0.4);}}
\newcommand{\glyphLate}{\tikz[baseline=0ex]{\fill[gray!30] (0,0) rectangle (0.12,0.15); \fill[gray!30] (0.14,0) rectangle (0.26,0.15); \fill[teal] (0.28,0) rectangle (0.4,0.45);}}
\newcommand{\glyphDiffuse}{\tikz[baseline=0ex]{\fill[purple!60] (0,0) rectangle (0.12,0.28); \fill[purple!60] (0.14,0) rectangle (0.26,0.28); \fill[purple!60] (0.28,0) rectangle (0.4,0.28);}}

\usepackage{amsmath}
\usepackage{amssymb}
\usepackage{mathtools}
\usepackage{amsthm}
\usepackage{comment}

\usepackage[capitalize,noabbrev]{cleveref}

\usepackage{enumitem}

\usepackage{float}
\usepackage{tcolorbox}
\tcbuselibrary{listings,breakable,skins}

\newtcolorbox{promptbox}[1][]{
  colback=gray!5,
  colframe=gray!70,
  fonttitle=\bfseries,
  coltitle=black,
  title=#1,
  breakable,
  enhanced,
  boxrule=0.5pt,
  arc=2pt,
  left=3pt,
  right=3pt,
  top=3pt,
  bottom=3pt
}
\theoremstyle{plain}
\newtheorem{theorem}{Theorem}[section]

\theoremstyle{definition}

\newtheorem{assumption}[theorem]{Assumption}
\theoremstyle{remark}
\newtheorem{remark}[theorem]{Remark}

\newtheorem*{theorem*}{Theorem}

\usepackage[textsize=tiny]{todonotes}
\usepackage{xcolor}

\begin{document}

\twocolumn[
  \begin{center}
    {\LARGE\bfseries From Text to Forecasts: Bridging Modality Gap with Temporal Evolution Semantic Space\par}
    \vspace{0.9em}
    {\large
      Lehui Li\textsuperscript{1,*},
      Yuyao Wang\textsuperscript{2,*},
      Jisheng Yan\textsuperscript{1},
      Wei Zhang\textsuperscript{1},
      Jinliang Deng\textsuperscript{3,4,\textdagger},
      Haoliang Sun\textsuperscript{1},
      Zhongyi Han\textsuperscript{1},
      Yongshun Gong\textsuperscript{1,\textdagger}\par
    }
    \vspace{0.6em}
    {\normalsize
      \textsuperscript{1}School of Software, Shandong University, China\par
      \textsuperscript{2}Boston University, Boston, MA, USA\par
      \textsuperscript{3}State Key Laboratory of Complex \& Critical Software Environment, Beihang University\par
      \textsuperscript{4}Research Institute of Trustworthy Autonomous Systems, SUSTech\par
    }
    \vspace{0.4em}
    {\normalsize
      Correspondence to:
      \href{mailto:jinliangdeng9588@gmail.com}{\texttt{jinliangdeng9588@gmail.com}},
      \href{mailto:ysgong@sdu.edu.cn}{\texttt{ysgong@sdu.edu.cn}}\par
    }
    \vspace{0.3em}
    {\normalsize \textsuperscript{*}Equal contribution. \textsuperscript{\textdagger}Co-corresponding authors.\par}
  \end{center}
  \vspace{1em}
  \begin{center}
    \begin{minipage}{0.95\textwidth}
      \small
      \textbf{Abstract.}
      Incorporating textual information into time-series forecasting holds promise for addressing event-driven non-stationarity; however, a fundamental modality gap hinders effective fusion: textual descriptions express temporal impacts implicitly and qualitatively, whereas forecasting models rely on explicit and quantitative signals. Through controlled semi-synthetic experiments, we show that existing methods over-attend to redundant tokens and struggle to reliably translate textual semantics into usable numerical cues. To bridge this gap, we propose \method{}, which introduces a Temporal Evolution Semantic Space as an intermediate bottleneck between modalities. This space consists of interpretable, numerically grounded temporal primitives---mean shift, volatility, shape, and lag---extracted from text by an LLM via structured prompting and filtered through confidence-aware gating. Experiments on four real-world datasets demonstrate up to a 29\% reduction in forecasting error compared to state-of-the-art uni-modal and multimodal baselines. The code will be released after acceptance.
    \end{minipage}
  \end{center}
  \vspace{1.2em}
]

\section{Introduction}

% \begin{figure}[!tbp]
%   \centering
%   \includegraphics[width=\columnwidth]{figures/motivation.png}
%   \vspace{-0.05in}
%   \caption{Illustration of the cross-modal transformation via the Temporal Semantic Space. Verbose textual narratives (left) are distilled into structured temporal primitives (middle), which then guide numerical forecasting (right).}
%   \label{fig:motivation}
% \end{figure}
\begin{figure}[!tbp]
  \centering
  \includegraphics[width=0.9\columnwidth]{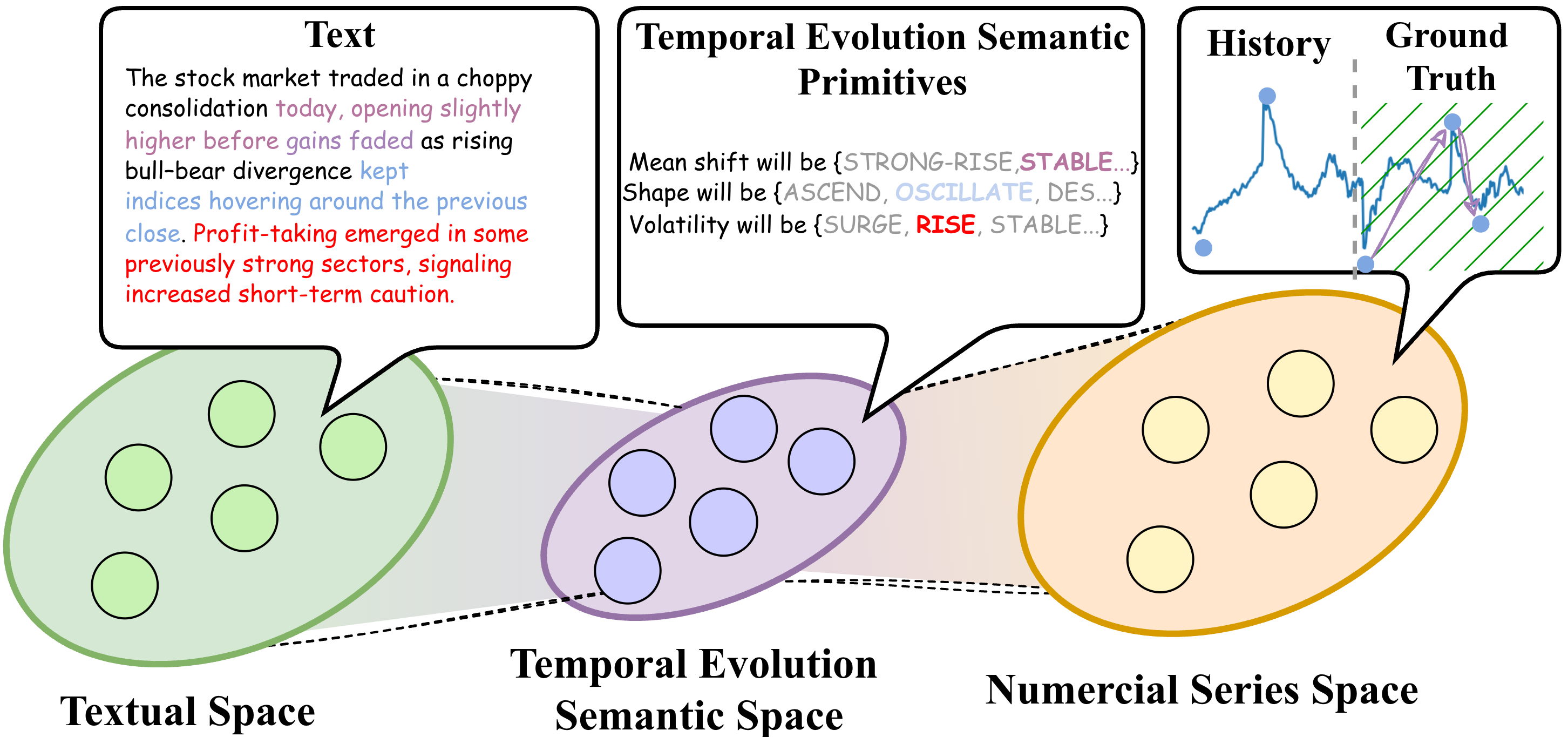}
  \vspace{-0.05in}
  \caption{Illustration of the cross-modal transformation via the Temporal Semantic Space. Verbose textual narratives (left) are distilled into structured temporal primitives (middle), which then guide numerical forecasting (right).}
  \label{fig:motivation}
\end{figure}

Time-series forecasting plays a critical role in diverse domains such as transportation, energy, and finance \citep{zhou2021informer,10.24963/ijcai.2025/371}. In recent years, unimodal forecasting models that rely solely on historical numerical observations have achieved substantial progress by modeling temporal dependencies \citep{wu2021autoformer}. However, real-world time series often exhibit significant non-stationarity \citep{du2021adarnn, kim2025battling}, where the underlying data-generating mechanism evolves over time. In particular, exogenous events such as accidents, extreme weather, or public sentiment shocks can trigger regime shifts \citep{osogami2021second,yang2025can}, causing statistical properties---including trends, volatility, mean, and variance---to change abruptly within short time windows \citep{10887343}. Under such conditions, models struggle to extract reliable predictive signals from historical numerical observations alone, leading to severe performance degradation \citep{zeng2023transformers}.

To mitigate event-driven non-stationarity, recent work has begun to incorporate textual data (e.g. news articles, social media posts) \citep{sawhney2020deep}, as an exogenous information source. The typical approach leverages pretrained language models to encode text into semantic embeddings \citep{niu2023kefvp}, which are then fused with numerical time-series features via concatenation, gating, or cross-modal attention \citep{sawhney2020deep,koval2025multimodal}. These multimodal methods have demonstrated notable gains on benchmarks with significant non-stationarity.

Despite recent progress, the substantial \emph{modality gap} \citep{liang2022mind} between time series data and textual data remains largely overlooked \citep{koval2025multimodal}, fundamentally limiting multimodal forecasting. Time series data is chronologically ordered and quantitative, offering precise measurements of temporal dynamics, yet it lacks explicit semantic abstraction. In contrast, textual data is semantically rich but unstructured and qualitative, with the impacts of events on temporal dynamics often implicit, diffuse, and weakly grounded in time. As a result, the predictive relevance of textual information is sparsely distributed across tokens and rarely aligned with the compact numerical structure required by time-series models. Through controlled semi-synthetic experiments, we analyze time-series predictors that directly fuse numerical observations with raw textual embeddings. We find that: (1) predictors fail to focus on prediction-relevant tokens, instead being distracted by noisy and redundant textual content;
(2) even after removing redundant information, models still struggle to correctly decode temporal evolution signals, indicating a fundamental representational mismatch between the two modalities.

To bridge this modality gap, As illustrated in Figure~\ref{fig:motivation}, we propose to construct an intermediate representation between the textual space and the numerical series space, termed the \emph{Temporal Evolution Semantic Space}. This space serves as an information bottleneck that explicitly filters, structures, and organizes textual information relevant to temporal evolution. Inspired by expert-driven paradigms in temporal pattern analysis, we predefine a set of temporal evolution primitives corresponding to critical attributes governing temporal dynamics, such as mean shift, volatility shift, evolution shape, and lag. Crucially, these primitives are expressed as verbal specifications, enabling large language models (LLMs) to emulate human-like judgments when interpreting textual descriptions of events. We decompose the fusion process into two stages: \cOne~\textbf{Text $\to$ Semantic Space}: We prompt an LLM to reason over the input text and explicitly label latent temporal evolution signals using the predefined primitives. Since such labels may be unreliable due to information insufficiency or limitations of the LLM, we introduce a learnable gating network to estimate and weight the reliability of the generated semantic labels. \cTwo~\textbf{Semantic Space $\to$ Numerical Space}: The inferred temporal evolution primitives are then injected as exogenous conditions into the time-series model, allowing it to leverage its numerical modeling and pattern recognition capabilities to ground semantic signals in observed temporal dynamics. Our \textbf{contributions} are summarized as follows:
\begin{itemize}[leftmargin=*,itemsep=2pt,topsep=2pt]
  \item Through semi-synthetic experiments, we identify two bottlenecks in text-time-series fusion: attention is distracted by redundant context, and qualitative expressions resist decoding into predictive gains.
  \item We propose \method{}, an intermediate space that distills text into temporal primitives via LLM extraction with confidence-aware gating, guiding forecasters as exogenous conditions.
  \item Experiments on four datasets show up to 29\% MSE reduction; analyses validate primitive complementarity and gating effectiveness across non-stationary patterns.
\end{itemize}

% \textbullet Through semi-synthetic experiments, we reveal two bottlenecks in text-time-series fusion: attention is distracted by redundant context, .

% \textbullet We propose \tesc{}, an intermediate space that distills text into temporal primitives via LLM extraction with confidence-aware gating

% \textbullet Experiments on four datasets show up to 29\% MSE reduction; analyses validate primitive complementarity 

\section{Related Work}

\textbf{Uni-Modal Time-Series Forecasting.}
Early time series forecasting relies on classical statistical methods that extrapolate trends under stationarity assumptions \cite{chen2004load,huang2003short,kalekar2004time}, performing well on smaller, simpler datasets. Deep learning models \cite{10.24963/ijcai.2025/371,deng2024disentangling,zhou2022fedformer} have since enabled end-to-end learning of nonlinear patterns and temporal dependencies in more complex settings. Nevertheless, these methods remain series-centric and inadequately account for external drivers and task context \cite{kim2022reversible, wang2024news}.

\paragraph{Multimodal Time-Series Forecasting.}

Multimodal time series forecasting integrates external text (e.g., news and announcements) \cite{timemmd,dong2024fnspid} with historical numerical sequences to better capture distribution drift and structural breaks in event-driven scenarios \cite{wang2024news}. Early approaches relied on shallow semantics such as sentiment scores or keyword, which often fail to reflect finer-grained temporal variations like mean shifts and volatility clustering. Recent advances in pretrained language models \cite{touvron2023llama} have motivated direct text-series fusion via cross-attention \cite{su2025textreinforcementmultimodaltime,timemmd} or LLMs \cite{wang2024news,chang2025time}. However, such methods suffer from narrative redundancy and the difficulty of translating implicit, non-quantitative expressions into stable predictive signals, leading to optimization instability and performance degradation.

\section{Preliminaries}
\subsection{Problem Formulation} 

We consider the multimodal time-series forecasting task. Let $\mathbf{x}_t \in \mathbb{R}^d$ denote the numerical observation at time step $t$, and let $s_t$ represent the temporally aligned textual information (e.g., news articles, announcements). Given a historical window of length $L$, the numerical input is defined as 
\begin{equation}
\mathbf{X}_{\text{time}} = [\mathbf{x}_{t-L+1}, \ldots, \mathbf{x}_t] \in \mathbb{R}^{L \times d},
\end{equation}
and the textual input is $\mathbf{X}_{\text{text}} = s_t$. The objective is to predict the numerical sequence over the next $H$ steps:
\begin{equation}
\hat{\mathbf{Y}}_t = [\hat{\mathbf{x}}_{t+1}, \ldots, \hat{\mathbf{x}}_{t+H}] = f_\theta(\mathbf{X}_{\text{time}}, \mathbf{X}_{\text{text}}),
\end{equation}
where $f_\theta$ denotes a multimodal forecasting model. The prevalent fusion paradigm employs modality-specific encoders each modality into latent representations $\mathbf{H}_{\text{time}} = \phi_{\text{time}}(\mathbf{X}_{\text{time}}) \in \mathbb{R}^{N \times d_1}$ and $\mathbf{H}_{\text{text}} = \phi_{\text{text}}(\mathbf{X}_{\text{text}}) \in \mathbb{R}^{M \times d_2}$, where $M$ is the number of text tokens, which are integrated via cross-modal attention or other fusion mechanisms.

\subsection{Diagnosing the Modality Gap}
\paragraph{Semi-Synthetic Benchmark Construction}

\begin{figure}[!tbp]
  \centering
  \includegraphics[width=\columnwidth]{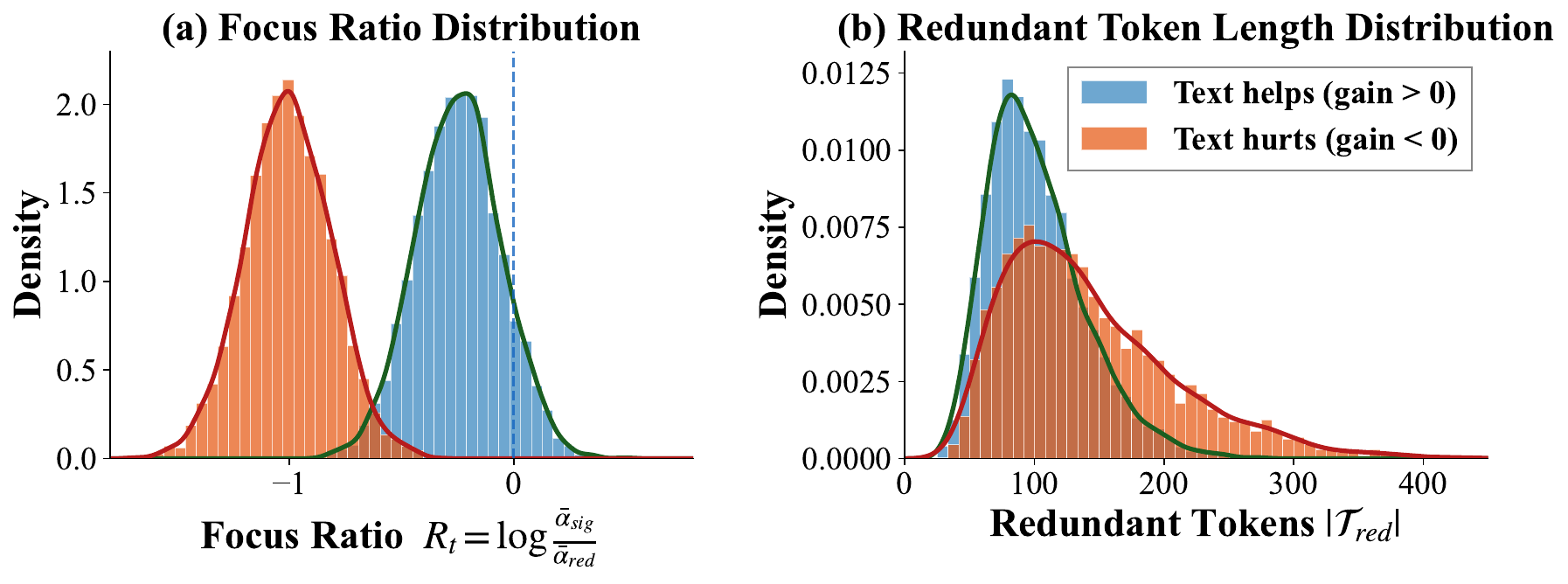}
  \vspace{-0.05in}
  \caption{Analysis of attention misalignment. \textbf{\textit{Left}}: Distribution of focus ratio $R_t$ on test samples. \textbf{\textit{Right}}: Relationship between redundant token count and predictive performance.}
  \label{fig:e1}
\end{figure}

\begin{figure}[!tbp]
  \centering
  \includegraphics[width=\columnwidth]{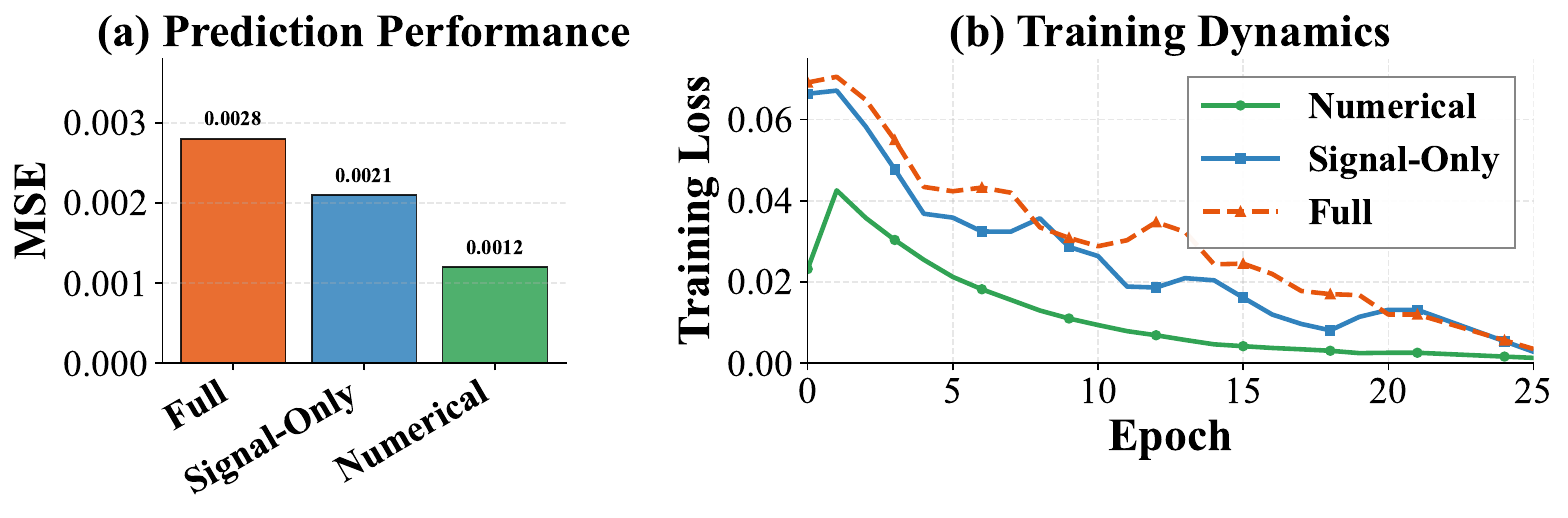}
  \vspace{-0.05in}
  \caption{Comparison of three input variants. \textbf{\textit{Left}}: Prediction performance (MSE) across Full, Signal-Only, and Numerical inputs. \textbf{\textit{Right}}: Training loss curves showing convergence dynamics.}
  \label{fig:e2}
\end{figure}
To investigate how the modality gap affects fusion, we construct a semi-synthetic benchmark where text is guaranteed to contain predictive signals, enabling controlled evaluation of information utilization. Specifically, we select FNSPID \citep{dong2024fnspid}, a financial dataset exhibiting pronounced non-stationarity, and extract real numerical sequences as the temporal modality input $\mathbf{X}_{\text{time}}$. For each sample, we further extract a set of statistical features characterizing the evolution patterns from its future window, and leverage GPT-5.2 \citep{singh2025openai} to map these features into natural language $\mathbf{X}_{\text{text}}$ temporally aligned with the sample. 

The traceable generation process enables automatic token-level annotation: tokens encoding future statistical features form the signal set $\mathcal{T}_{\text{sig}}$, while context-only tokens form the redundant set $\mathcal{T}_{\text{red}}$. Since prediction-relevant information is typically sparsely distributed within natural language, we generally have $|\mathcal{T}_{\text{sig}}| \ll |\mathcal{T}_{\text{red}}|$. This design preserves the naturalistic properties of text while introducing controlled interference, enabling us to quantitatively assess in a unified and interpretable setting: when text already contains effective information, whether existing fusion mechanisms can reliably localize $\mathcal{T}_{\text{sig}}$, suppress interference from $\mathcal{T}_{\text{red}}$, and genuinely translate critical semantics into predictive gains.

\paragraph{Analysis of Attention Distraction} 
We first examine whether the fusion mechanism can correctly identify predictive signals within text. To quantify the rationality of attention allocation, we define the focus ratio:
\begin{align}
R_t &= \log \frac{\bar{\alpha}_{\text{sig}}}{\bar{\alpha}_{\text{red}}}, \label{eq:focus-ratio}\\
\bar{\alpha}_{\text{sig}} &= \frac{1}{|\mathcal{T}_{\text{sig}}|}{\textstyle\sum\nolimits_{i \in \mathcal{T}_{\text{sig}}}} \alpha_{t,i},\;
\bar{\alpha}_{\text{red}} = \frac{1}{|\mathcal{T}_{\text{red}}|}{\textstyle\sum\nolimits_{j \in \mathcal{T}_{\text{red}}}} \alpha_{t,j}, \nonumber
\end{align}
Here $\alpha_{t,i}$ denotes the cross-attention weight on token $i$. If the model correctly localizes predictive signals, we expect $R_t > 0$; conversely, $R_t < 0$ indicates that attention is dominated by redundant information. \cref{fig:e1} (\textbf{\textit{Left}}) presents the distribution of $R_t$ on test samples: even among samples where text yields positive predictive gains, the vast majority exhibit $R_t < 0$, indicating that the model systematically over-attends to redundant tokens rather than predictive signals. \cref{fig:e1} (\textbf{\textit{Right}}) corroborates this finding: samples with negative gains exhibit significantly more redundant tokens, confirming that redundancy hampers textual information extraction.% Therefore, \paragraph{Redundant Descriptions Cause Attention Misalignment.}

\paragraph{Analysis of Representational Mismatch}
%\paragraph{Implicit, Non-Quantifiable Descriptions Fail to Translate into Predictive Signals.}
The preceding analysis demonstrates that redundant tokens interfere with signal extraction. A natural question arises: \textit{if redundancy is entirely removed, can the model effectively utilize predictive information in text?} To this end, we construct three variants: \textsc{Full} (complete text), \textsc{Signal-Only} (only $\mathcal{T}{\text{sig}}$), and \textsc{Numerical} (statistical features as exogenous inputs). \cref{fig:e2} (\textbf{\textit{Left}}) presents the prediction performance comparison across the three variants. \textsc{Signal-Only} achieves lower MSE than \textsc{Full}; however, its performance remains significantly inferior to \textsc{Numerical}, indicating that even when redundancy is completely removed, textual signals still cannot be effectively translated into predictive gains. \cref{fig:e2} (\textbf{\textit{Right}}) further corroborates this finding through training curves: \textsc{Numerical} converges rapidly with smooth trajectories, \textsc{Signal-Only} exhibits pronounced convergence lag and loss oscillation, while \textsc{Full} displays the most unstable optimization dynamics. These observations collectively reveal the fundamental bottleneck of cross-modal transformation---natural language tends to characterize temporal evolution in implicit, qualitative terms (e.g., ``significant rise'' rather than ``+15.3\%''), and such semantics are difficult for existing fusion mechanisms to decode into quantitative signals usable for numerical prediction.

\section{Methodology}

As shown in Figure~\ref{fig:algorithm}, we propose \method{}, a two-stage framework built around a \emph{Temporal Evolution Semantic Space} that distills natural language into quantifiable temporal primitives for stable cross-modal fusion. We first define this space (\cref{sec:semantic-space}), then introduce confidence-aware text-to-primitive projection (\cref{sec:stage1}), and finally present primitive-conditioned forecasting (\cref{sec:stage2}).

We cast text-augmented multimodal forecasting as a semantic bottleneck: instead of directly fusing token-level features
$\mathbf{H}_{\text{text}}=\phi_{\text{text}}(\mathbf{X}_{\text{text}})$,
we compress $\mathbf{X}_{\text{text}}$ into a compact set of temporal primitives and let the forecaster condition on these primitives together with
$\mathbf{H}_{\text{time}}=\phi_{\text{time}}(\mathbf{X}_{\text{time}})$.
Theorem~\ref{thm:main} specifies when this constraint preserves predictive information and improves generalization by reducing dependence on token-level variation.

% We view multimodal forecasting with text as a semantic bottleneck problem: instead of letting $f_\theta$ directly fuse token-level text representations $\mathbf{H}_{\text{text}}=\phi_{\text{text}}(\mathbf{X}_{\text{text}})$,we first distill $\mathbf{X}_{\text{text}}$ into a compact set of temporal primitives and let the forecaster condition on these primitives together with $\mathbf{H}_{\text{time}}=\phi_{\text{time}}(\mathbf{X}_{\text{time}})$.The theorem \ref{thm:main} below formalizes when this restriction preserves predictive information and can improve generalization by reducing dependence on spurious token-level variation.

\begin{theorem}\label{thm:main}
Assume semantic sufficiency (Assump.~\ref{assump:semantic-suff}) in our forecasting setting:
$\hat{\mathbf{Y}}_t \perp\!\!\!\perp \mathbf{X}_{\text{text}} \mid (\mathbf{P}_t,\mathbf{X}_{\text{time}})$,
where $\mathbf{P}_t$ denotes the distilled primitives.
Then $\forall$ encoders $f:\mathbf{X}_{\text{text}}\!\mapsto Z$, $\exists$ an encoder $\tilde f:\mathbf{P}_t\!\mapsto \tilde Z$ s.t.
\[
I(\tilde Z;\hat{\mathbf{Y}}_t\mid \mathbf{X}_{\text{time}})
= I(Z;\hat{\mathbf{Y}}_t\mid \mathbf{X}_{\text{time}}),
\]
\[
I(\tilde Z;\mathbf{X}_{\text{text}})
\le I(Z;\mathbf{X}_{\text{text}}).
\]

Moreover, under standard sub-Gaussian loss and i.i.d.\ sampling Assumption ~\ref{ass:subgaussian}, and ~\ref{ass:independent-samples},
\[
\textsc{Gen}(\tilde Z)\le \textsc{Gen}(Z),
\]
where $\textsc{Gen}(\cdot)$ denotes the expected population--empirical loss gap of a predictor trained on the representation.
\end{theorem}

\noindent\emph{Proof.} See Appendix~\ref{app:theory}

% Here, we motivate by viewing multimodal forecasting with text as a semantic bottleneck problem:
% instead of letting the forecaster rely on token-level text $X$, we force it to depend on a compact semantic interface
% $V=\phi(X)$ together with the time-series context $S$.
% Below we state a theorem showing that, under a mild semantic sufficiency condition, this restriction incurs no loss in
% predictive information and can tighten generalization by reducing dependence on spurious token-level variation.

\begin{figure*}[t]
  \centering
  \includegraphics[width=0.9\textwidth]{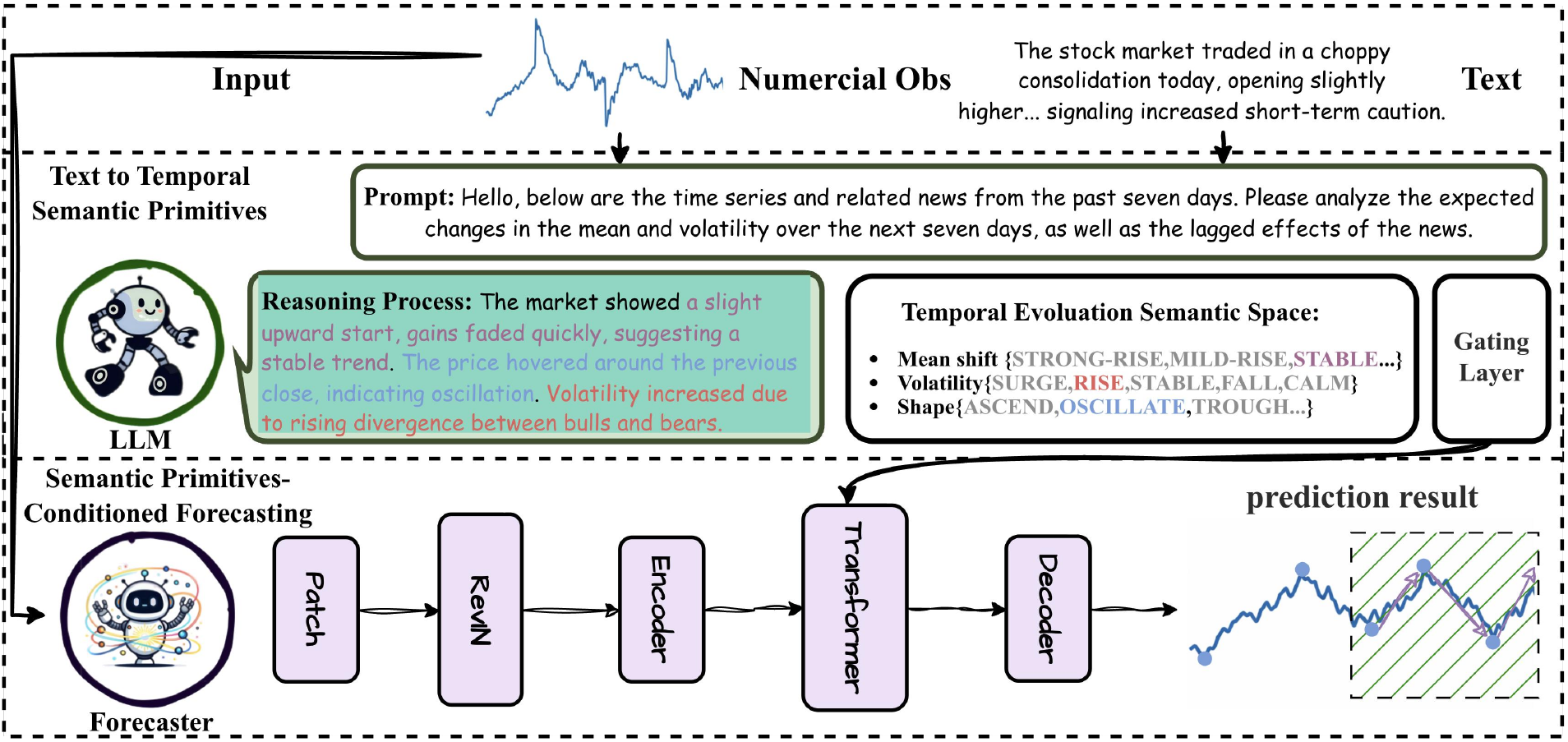}
  \caption{Overview of \method{}. Given numerical observations and associated text, a frozen LLM extracts temporal evolution primitives (e.g., mean shift, shape, volatility, lag) via structured prompting. These primitives, after confidence-aware gating, condition a Transformer-based forecaster that fuses semantic signals with encoded historical sequences to produce numerical predictions.}
  \label{fig:algorithm}
\end{figure*}

\subsection{Temporal Evolution Semantic Space}
\label{sec:semantic-space}

To overcome the limitations of directly mapping text to numerical sequences, we introduce an intermediate representation---the \emph{Temporal Evolution Semantic Space}---that mediates between textual and numerical modalities. The design rationale derives from established practices in domain-specific forecasting: rather than processing descriptive details verbatim, practitioners extract implications for temporal statistical properties. Guided by this insight, we distill three categories of statistically critical features from the classical forecasting literature and formalize them as \emph{Temporal Semantic Primitives} (TSPs). The first two primitives characterize \emph{what} aspects of the time series evolve, while the third captures \emph{when} and \emph{for how long} these impacts manifest. %$\mathcal{P} = \{p_1, \ldots, p_K\}$, where each primitive $p_k$ corresponds to a discrete value domain $\mathcal{V}_k$ and a deterministic mapping $\psi_k$.We instantiate $K=4$ primitives as follows.

\textbf{(1)~Distribution Shift Primitives.}
Distribution shift---temporal changes in statistical properties---is identified by \citet{kim2022reversible} as a principal source of forecasting degradation in non-stationary time series. We characterize such shifts using two complementary primitives capturing changes in mean level and volatility. Let $\mathbf{X}_t \in \mathbb{R}^L$ and $\mathbf{Y}_t \in \mathbb{R}^H$ denote the observation and forecast windows. The mean shift is quantified by the standardized difference $\Delta_\mu = (\bar{Y}_t - \bar{X}_t)/\sigma(X_t)$, where $\bar{X}_t = \tfrac{1}{L}\sum_i x_i$ and $\bar{Y}_t = \tfrac{1}{H}\sum_j y_j$. 

Volatility shift is measured by $r_\sigma = \log [(\sigma_Y + \epsilon)/(\sigma_X + \epsilon)]$, where $\sigma_X = \text{std}(\Delta \mathbf{X}_t)$, $\sigma_Y = \text{std}(\Delta \mathbf{Y}_t)$, and $\Delta$ denotes the first-order differencing operator. Both shift measures are discretized using adaptive thresholds $\tau_1 < \tau_2$ defined by training-set quantiles.

\begin{comment}

Distribution shift---the temporal evolution of statistical properties---is identified by \citet{kim2022reversible} as a principal source of forecasting degradation in non-stationary time series. We capture this phenomenon through two complementary primitives that characterize changes in \emph{mean level} and \emph{volatility}, respectively. Let $\mathbf{X}_t \in \mathbb{R}^L$ and $\mathbf{Y}_t \in \mathbb{R}^H$ denote the observation and forecast windows.
 
\emph{Mean Shift Primitive $p_{\text{mean}}$.} This primitive captures both the \emph{direction} and \emph{magnitude} of level changes between windows. We compute window means $\bar{X}_t = \frac{1}{L}\sum_{i} x_i$ and $\bar{Y}_t = \frac{1}{H}\sum_{j} y_j$, and define the standardized shift $\Delta_\mu = (\bar{Y}_t - \bar{X}_t) / \sigma(X_t)$. Partitioning via thresholds $\tau_1 < \tau_2$ yields:
\end{comment}
\vspace{-0.3em}
\begin{center}
\resizebox{0.95\columnwidth}{!}{%
\begin{tabular}{@{}clll@{}}
\toprule
& \textbf{Category} & \textbf{Condition} & \textbf{Semantics} \\
\midrule
\glyphStrongRise & \textsc{strong-rise} & $\Delta_\mu (r_\sigma) > \tau_2$ & Significant increase \\
\glyphMildRise & \textsc{mild-rise} & $\tau_1 < \Delta_\mu (r_\sigma) \le \tau_2$ & Moderate increase \\
\glyphStable & \textsc{stable} & $-\tau_1 \le  \Delta_\mu (r_\sigma) \le \tau_1$ & Level unchanged \\
\glyphMildDrop & \textsc{mild-drop} & $-\tau_2 \le \Delta_\mu (r_\sigma) < -\tau_1$ & Moderate decrease \\
\glyphStrongDrop & \textsc{strong-drop} & $\Delta_\mu (r_\sigma) < -\tau_2$ & Significant decrease \\
\bottomrule
\end{tabular}}
\end{center}
\vspace{0.3em}

\begin{comment}
    \emph{Volatility Primitive $p_{\text{vol}}$.} Non-stationary series frequently exhibit pronounced heteroskedasticity: volatility clustering \citep{zhao2024garch} and regime switching \citep{ansari2021redsds} alter the uncertainty structure of future observations. This primitive captures regime transitions by comparing realized volatility between windows. We compute first-difference standard deviations $\sigma_X = \text{std}(\Delta \mathbf{X}_t)$ and $\sigma_Y = \text{std}(\Delta \mathbf{Y}_t)$, where $\Delta x_i = x_i - x_{i-1}$, and define the relative volatility shift $r_\sigma = \log\bigl((\sigma_Y + \epsilon)/(\sigma_X + \epsilon)\bigr)$:
\vspace{-0.3em}
\begin{center}
\resizebox{0.95\columnwidth}{!}{%
\begin{tabular}{@{}clll@{}}
\toprule
& \textbf{Category} & \textbf{Condition} & \textbf{Semantics} \\
\midrule
\glyphSurge & \textsc{surge} & $r_\sigma > \tau_2$ & Volatility spike \\
\glyphVolRise & \textsc{rise} & $\tau_1 < r_\sigma \le \tau_2$ & Volatility increase \\
\glyphVolStable & \textsc{stable} & $|r_\sigma| \le \tau_1$ & Volatility unchanged \\
\glyphVolFall & \textsc{fall} & $-\tau_2 \le r_\sigma < -\tau_1$ & Volatility decrease \\
\glyphCalm & \textsc{calm} & $r_\sigma < -\tau_2$ & Volatility collapse \\
\bottomrule
\end{tabular}}
\end{center}
\vspace{0.3em}

For both primitives, thresholds $\tau_1 < \tau_2$ are set as training-set quantiles to ensure adaptive categorization.
\end{comment}

\textbf{(2)~Shape Primitive $p_{\text{shape}}$.}
%While mean shift captures overall level changes
\citet{nie2022time} demonstrate that patch-level morphology conveys rich predictive information. The shape primitive encodes the \emph{inter-patch trend sequence} to characterize internal evolution structure. Concretely, we partition $\mathbf{Y}_t = (y_1, \ldots, y_H)$ into $N_{\text{fcst}}$ equally-sized patches of length $L_u = H/N_{\text{fcst}}$. For each patch $i$, we compute its mean $\bar{u}_i = \frac{1}{L_u}\sum_{j=(i-1)L_u+1}^{iL_u} y_j$ and define inter-patch trend signs $s_i = \text{sgn}_\tau(\bar{u}_{i+1} - \bar{u}_i)$, where $\text{sgn}_\tau(x) = +1$ if $x > \tau$, $-1$ if $x < -\tau$, and $0$ otherwise. The dominant pattern of $(s_1, \ldots, s_{N_{\text{fcst}}-1})$ determines the shape category:
\vspace{-0.3em}
\begin{center}
\resizebox{0.95\columnwidth}{!}{%
\begin{tabular}{@{}clll@{}}
\toprule
& \textbf{Category} & \textbf{Condition} & \textbf{Semantics} \\
\midrule
\glyphAscend & \textsc{ascend} & $\forall i{:}\, s_i \!\ge\! 0 \land \exists i{:}\, s_i \!=\! {+}1$ & Sustained uptrend \\
\glyphDescend & \textsc{descend} & $\forall i{:}\, s_i \!\le\! 0 \land \exists i{:}\, s_i \!=\! {-}1$ & Sustained downtrend \\
\glyphPeak & \textsc{peak} & rise-then-fall with sign changes & Rise then fall \\
\glyphTrough & \textsc{trough} & fall-then-rise with sign changes & Fall then rise \\
\glyphOscillate & \textsc{oscillate} & multiple sign reversals & Fluctuating \\
\bottomrule
\end{tabular}}
\end{center}
\vspace{0.3em}

\textbf{(3)~Lag and Decay Primitive $p_{\text{lag}}$.}
The same mean increase or volatility surge may exhibit vastly different response dynamics (transient pulse vs.\ long-tailed decay). Inspired by distributed-lag models and impulse response functions \citep{YangVCformer2024}, we propose the lag-and-decay primitive to localize influence timing and persistence within the forecast horizon. Using the same patch partition, we compute influence intensity for each patch: $a_i = |(\bar{u}_i - \bar{X}_t)/\sigma(X_t)| + \alpha |\log((\sigma(\Delta \mathbf{u}_i) + \epsilon)/(\sigma(\Delta \mathbf{X}_t) + \epsilon))|$, and normalize to obtain an influence distribution $\pi_i = a_i / \sum_j a_j$. We then derive three indicators: (i)~centroid $c = \sum_i \pi_i \cdot (i{-}1)/(N_{\text{fcst}}{-}1) \in [0,1]$, where smaller $c$ indicates earlier onset; (ii)~tail mass $d = \sum_{i > i^*} \pi_i$ (with $i^* = \arg\max_i \pi_i$), where larger $d$ indicates stronger persistence; (iii)~peak prominence $q = \max_i \pi_i$, distinguishing concentrated from diffuse effects. Given thresholds $\kappa_1 < \kappa_2$ (timing), $\rho$ (persistence), and $\eta$ (prominence):
\vspace{-0.3em}
\begin{center}
\resizebox{0.95\columnwidth}{!}{%
\begin{tabular}{@{}clll@{}}
\toprule
& \textbf{Category} & \textbf{Condition} & \textbf{Semantics} \\
\midrule
\glyphEarlyFade & \textsc{early-fade} & $q \!>\! \eta,\, c \!\le\! \kappa_1,\, d \!\le\! \rho$ & Early onset, quick decay \\
\glyphEarlyPersist & \textsc{early-persist} & $q \!>\! \eta,\, c \!\le\! \kappa_1,\, d \!>\! \rho$ & Early onset, persistent \\
\glyphMidFade & \textsc{mid-fade} & $q \!>\! \eta,\, \kappa_1 \!<\! c \!\le\! \kappa_2,\, d \!\le\! \rho$ & Mid-horizon pulse \\
\glyphMidPersist & \textsc{mid-persist} & $q \!>\! \eta,\, \kappa_1 \!<\! c \!\le\! \kappa_2,\, d \!>\! \rho$ & Mid-horizon, persistent \\
\glyphLate & \textsc{late} & $q \!>\! \eta,\, c \!>\! \kappa_2$ & Late manifestation \\
\glyphDiffuse & \textsc{diffuse} & $q \le \eta$ & Distributed impact \\
\bottomrule
\end{tabular}}
\end{center}
\vspace{0.3em}

The above design ensures that each primitive $p_k$ possesses \emph{numerical verifiability}: given observation and forecast sequences, the ground-truth value $v_{t,k}$ is uniquely determined via $\psi_k$, thereby furnishing reliable supervision for the gating mechanism. All thresholds are set adaptively based on training-set statistics (e.g., quantiles).

% Theoretically, when the task-relevant content of the raw text can be captured by a semantic variable
% $V=\phi(X)$ such that $Y \perp X \mid (V,S)$ (Assumption~\ref{assump:semantic-suff}),
% one can restrict attention to semantic encoders that depend on $X$ only through $V$ without losing
% predictive information about $Y$ conditional on $S$.
% In particular, Theorem~\ref{thm:ib-restriction} shows such a restriction preserves $I(Z;Y\mid S)$
% while not increasing $I(Z;X)$, and Theorem~\ref{thm:generalization} further implies this can only
% tighten the expected generalization bound under Assumptions~\ref{ass:subgaussian}--\ref{ass:independent-samples}.

\subsection{Text to Temporal Semantic Primitives}
\label{sec:stage1}

This stage describes how to reliably predict the discrete value of each semantic primitive $p_k$ from textual input $s_t$, and estimate its confidence to filter noisy extractions.

\paragraph{LLM-based Primitive Classification.}
Given the finite value domain $\mathcal{V}_k$ of each primitive, we formulate extraction as multi-class classification using a frozen LLM as the classifier. Under a structured prompt $\mathcal{D}_k$ (containing primitive semantic definitions and representative examples), we compute log-likelihood scores for each candidate $v \in \mathcal{V}_k$:
\begin{equation}
  \ell_{t,k}(v) = \log P_{\text{LLM}}(v \mid s_t, \mathcal{D}_k).
  \label{eq:llm-score}
\end{equation}
Applying temperature-scaled softmax yields a categorical distribution and the predicted class:
\begin{align}
  q_{t,k}(v) &= \frac{\exp(\ell_{t,k}(v)/T)}{\sum_{v'} \exp(\ell_{t,k}(v')/T)}, \label{eq:primitive-dist} \\
  \hat{v}_{t,k} &= \arg\max_{v \in \mathcal{V}_k} q_{t,k}(v), \label{eq:primitive-pred}
\end{align}
where $T$ is a temperature coefficient. This procedure yields both the discrete prediction $\hat{v}_{t,k}$ and the full distribution $q_{t,k}(\cdot)$, which provides calibration signals for subsequent confidence estimation.

\paragraph{Confidence-Aware Gating.}
Although LLMs can discriminate based on primitive definitions, real-world text often contains narrative noise and ambiguous expressions that may lead to extraction errors. To mitigate the negative impact of erroneous primitives on downstream predictions \citep{10.24963/ijcai.2025/371}, we introduce a confidence gate $g_{t,k} \in [0,1]$ for each primitive to suppress unreliable semantic injection during inference.

Since each primitive's candidate set $\mathcal{V}_k$ is small, we adopt a simple uncertainty indicator as the calibration signal: the log-probability margin between the top-1 and top-2 candidates. Let $v^{(1)}, v^{(2)}$ denote the classes with highest and second-highest probability under $q_{t,k}(\cdot)$:
\begin{equation}
  m_{t,k} = \log q_{t,k}(v^{(1)}) - \log q_{t,k}(v^{(2)}).
  \label{eq:margin}
\end{equation}
We map the predicted class $\hat{v}_{t,k}$ to an internal semantic vector $\mathbf{h}_{t,k}$ via a learnable embedding matrix $\mathbf{E}_k \in \mathbb{R}^{|\mathcal{V}_k| \times d}$:
\begin{equation}
  \mathbf{h}_{t,k} = \mathbf{E}_k[\hat{v}_{t,k}],
  \label{eq:primitive-embed}
\end{equation}
where each row of $\mathbf{E}_k$ corresponds to one discrete class of primitive $k$, learned end-to-end with the downstream predictor. The gating network fuses the semantic embedding with the margin to estimate confidence:
\begin{equation}
  g_{t,k} = \sigma\bigl(\mathbf{w}_k^\top [\mathbf{h}_{t,k};\, \mathbf{W}_m m_{t,k}] + b_k\bigr),
  \label{eq:gating}
\end{equation}
where $\mathbf{W}_m \in \mathbb{R}^{d}$ projects the scalar margin to the embedding dimension.
Thanks to the numerical verifiability of primitives (\cref{sec:semantic-space}), we obtain ground-truth labels from the forecast window $v^*_{t,k} = \psi_k(\mathbf{Y}_t)$ and construct supervision labels $y_{t,k} = \mathbb{1}[\hat{v}_{t,k} = v^*_{t,k}]$. The gating network is trained with binary cross-entropy:
\begin{equation}
  \mathcal{L}_{\text{gate}} = -\sum_{t,k} \bigl[ y_{t,k} \log g_{t,k} + (1 {-} y_{t,k}) \log (1 {-} g_{t,k}) \bigr].
  \label{eq:gate-loss}
\end{equation}
During inference, we apply soft weighting $\tilde{\mathbf{h}}_{t,k} = g_{t,k} \cdot \mathbf{h}_{t,k}$, allowing high-confidence primitives to dominate predictions while suppressing interference from erroneous extractions. 

% Theoretically, under a mild Lipschitz assumption on the forecaster with respect to the gated primitives,
% the prediction error induced by an incorrect primitive is attenuated proportionally to $g_{t,k}^2$
% (Theorem~\ref{thm:soft_gating_stability}).

\subsection{Semantic Primitives-Conditioned Forecasting}
\label{sec:stage2}

We feed the gated primitives $\tilde h_{t,k}=g_{t,k}h_{t,k}$ into the forecaster to produce $\hat y_t$.
Under a mild Lipschitz assumption on the forecaster, the error induced by an incorrect primitive is attenuated on the order of $g_{t,k}^2$ (Theorem~\ref{thm:soft_gating_stability}). We adopt PatchTST~\cite{nie2022time} as the backbone architecture and introduce a semantic conditioning mechanism.

% This stage fuses numerical time series with gated semantic primitives to generate future predictions. We adopt PatchTST~\cite{nie2022time} as the backbone architecture and introduce a semantic conditioning mechanism.

\paragraph{Time-Series Encoding.}
Following PatchTST \citep{nie2022time}, we first apply instance normalization \citep{kim2022reversible} to the input sequence $\mathbf{x} \in \mathbb{R}^{L}$ to mitigate distribution shift between training and testing: for each instance we compute mean $\mu$ and standard deviation $s$, then normalize before patching. The input is segmented into $N = \lfloor (L - P) / S \rfloor + 1$ patches, where $P$ is the patch length and $S$ is the stride. Each patch is mapped to the $d$-dimensional latent space via a trainable linear projection $\mathbf{W}_p \in \mathbb{R}^{d \times P}$, with learnable positional encodings $\mathbf{W}_{\text{pos}} \in \mathbb{R}^{d \times N}$ added to obtain patch embeddings $\mathbf{E}_{\text{patch}} \in \mathbb{R}^{N \times d}$.

\paragraph{Semantic Primitives-Conditioned Prediction.}
The $K$ gated semantic vectors $\tilde{\mathbf{h}}_{t,k} \in \mathbb{R}^d$ from \cref{sec:stage1} are stacked to form semantic prefix tokens $\mathbf{P} \in \mathbb{R}^{K \times d}$, concatenated with patch embeddings to form the Transformer input:
\begin{equation}
  \mathbf{Z}^{(0)} = [\mathbf{P};\, \mathbf{E}_{\text{patch}}] \in \mathbb{R}^{(K+N) \times d}.
  \label{eq:prefix-input}
\end{equation}
The input sequence is processed by $M$ Transformer encoder layers. Each layer first captures sequential dependencies via $H$-head self-attention:
\begin{equation}
  \text{Attn}(\mathbf{Q}, \mathbf{K}, \mathbf{V}) = \text{Softmax}\left(\frac{\mathbf{Q} \mathbf{K}^\top}{\sqrt{d/H}}\right) \mathbf{V},
  \label{eq:attention}
\end{equation}
where $\mathbf{Q}, \mathbf{K}, \mathbf{V} \in \mathbb{R}^{(K+N) \times (d/H)}$ are obtained via linear projections. Multi-head outputs are concatenated and updated through feed-forward networks with residual connections. This prefix fusion allows semantic information to participate in temporal modeling throughout the attention mechanism. Finally, we extract patch-corresponding outputs $\mathbf{Z}_{\text{out}} \in \mathbb{R}^{N \times d}$, flatten and map to the forecast horizon via an MLP, and apply inverse normalization to recover the original scale:
\begin{equation}
  \hat{\mathbf{y}} = s \cdot \text{MLP}(\text{Flatten}(\mathbf{Z}_{\text{out}})) + \mu \in \mathbb{R}^{H}.
  \label{eq:prediction-head}
\end{equation}

\paragraph{Training Objective.}
The model is trained end-to-end (with the LLM frozen) using MSE loss to measure prediction-target discrepancy, jointly optimized with gating supervision:
\begin{equation}
  \mathcal{L} = \mathcal{L}_{\text{fcst}} + \lambda \mathcal{L}_{\text{gate}}, \quad \mathcal{L}_{\text{fcst}} = \frac{1}{H}\|\hat{\mathbf{y}} - \mathbf{y}\|_2^2.
  \label{eq:loss}
\end{equation}

\section{Experiments}
\label{sec:experiments}
\begin{table*}[t]
  \centering
  \caption{Comparison of baselines for time-series forecasting. Lower scores indicate better performance. \best{Red}: best, \second{Blue}: runner-up.}
  \label{tab:comparison}
  {
  \setlength{\tabcolsep}{4pt}
  \renewcommand{\arraystretch}{1.15}
  \resizebox{\textwidth}{!}{
  \begin{tabular}{l|ccc|ccc|ccc|ccc}
  \toprule
  \multirow{2}{*}{\textbf{Model}} &
    \multicolumn{3}{c|}{\textbf{Bitcoin} (Finance)} &
    \multicolumn{3}{c|}{\textbf{FNSPID} (Finance)} &
    \multicolumn{3}{c|}{\textbf{Electricity} (General)} &
    \multicolumn{3}{c}{\textbf{Environment} (General)} \\
   & MAE & MSE & RMSE & MAE & MSE & RMSE & MAE & MSE & RMSE & MAE & MSE & RMSE \\
  \midrule
  \multicolumn{13}{l}{\textbf{Traditional Models}} \\
  \addlinespace[0.3em]
  TimeMixer & 1.6757 & 4.3725 & 2.0910 & 0.0153 & 0.0017 & 0.0407 & 0.1064 & 0.0252 & 0.1586 & 0.4219 & 0.3693 & 0.6077 \\
  \rowcolor{gray!05}
  TSMixer & 4.1202 & 29.1934 & 5.4031 & 0.0353 & 0.0082 & 0.0907 & 0.1119 & 0.0258 & 0.1606 & 0.4860 & 0.3878 & 0.6227 \\
  Nonstationary  & 1.6003 & 4.3420 & 2.0837 & 0.0154 & 0.0016 & 0.0400 & 0.1066 & 0.0256 & 0.1600 & \best{0.4178} & \best{0.3472} & \best{0.5892} \\
  \rowcolor{gray!05}
  TimesNet & 1.4598 & 3.8229 & 1.9552 & \second{0.0152} & \second{0.0015} & 0.0385 & 0.1035 & \second{0.0242} & \second{0.1557} & 0.4258 & 0.3473 & \second{0.5893} \\
  FEDformer & 1.4197 & 3.2768 & 1.8102 & 0.0163 & 0.0016 & 0.0396 & 0.1144 & 0.0267 & 0.1633 & 0.4513 & 0.3650 & 0.5893 \\
  \rowcolor{gray!05}
  Pyraformer & 1.3994 & 3.6377 & 1.9073 & 0.0168 & 0.0015 & \second{0.0383} & 0.1117 & 0.0261 & 0.1616 & 0.4303 & 0.3649 & 0.6041 \\
  Reformer & 12.4630 & 157.8843 & 12.5652 & 0.0299 & 0.0141 & 0.1188 & 0.1180 & 0.0292 & 0.1709 & 0.4258 & 0.3473 & 0.5893 \\
  \rowcolor{gray!05}
  PatchTST & 1.3615 & 3.2456 & 1.8016 & 0.0160 & 0.0016 & 0.0402 & \second{0.1047} & 0.0243 & 0.1559 & 0.4290 & 0.3706 & 0.6088 \\
  \midrule
  \multicolumn{13}{l}{\textbf{Multimodal Models}} \\
  \addlinespace[0.3em]
  TimeLLM & 1.3940 & 3.4477 & 1.8568 & 0.0158 & 0.0017 & 0.0412 & 0.1074 & 0.0255 & 0.1596 & 0.4237 & 0.3714 & 0.6094 \\
  \rowcolor{gray!05}
  ChatTime & 1.4774 & 3.9389 & 1.9846 & 0.0187 & 0.0018 & 0.0424 & 0.1149 & 0.0279 & 0.1670 & 0.4344 & 0.3787 & 0.6153 \\
  NewsForecasting & \second{1.3593} & \second{3.2036} & \second{1.7898} & 0.0176 & 0.0017 & 0.0412 & 0.1088 & 0.0254 & 0.1593 & 0.4378 & 0.3678 & 0.6064 \\
  \cmidrule(lr){1-13}
  \rowcolor{gray!05}
  Ours & \best{1.1120} & \best{2.2726} & \best{1.5075} & \best{0.0147} & \best{0.0012} & \best{0.0347} & \best{0.1031} & \best{0.0230} & \best{0.1518} & \second{0.4202} & \second{0.3534} & 0.5945 \\
  \rowcolor{gray!02}
  \textit{Gain vs. best (\%)} & $+18.2\%$ & $+29.1\%$ & $+15.8\%$ & $+3.3\%$ & $+20.0\%$ & $+9.9\%$ & $+0.4\%$ & $+5.0\%$ & $+2.5\%$ & $-0.6\%$ & $-1.8\%$ & $-0.9\%$ \\
  \bottomrule
  \end{tabular}}
  }
\end{table*}
\begin{table}[t]
  \centering
  \small
  \caption{Ablation study on the contributions of temporal evolution semantic space (TESS) and gating mechanism. Lower scores indicate better performance. \best{Red}: best performance.}
  \label{tab:ablation}
  \setlength{\tabcolsep}{3pt}
  \renewcommand{\arraystretch}{0.9}
  \begin{tabular}{lcccc}
  \toprule
  \textbf{Dataset} & \textbf{Metric} & \textbf{w/o TESS} & \textbf{w/o Gating} & \textbf{\method{}} \\
  \midrule
  \multirow{3}{*}{\textbf{Bitcoin}} 
    & MAE & 1.6025 & 1.1532 & \best{1.1120} \\
    & MSE & 4.2238 & 2.3556 & \best{2.2726} \\
    & RMSE & 2.0552 & 1.5349 & \best{1.5075} \\
  \midrule
  \multirow{3}{*}{\textbf{FNSPID}} 
    & MAE & 0.0155 & 0.0151 & \best{0.0147} \\
    & MSE & 0.0018 & 0.0015 & \best{0.0012} \\
    & RMSE & 0.0412 & 0.0387 & \best{0.0347} \\
  \midrule
  \multirow{3}{*}{\textbf{Electricity}} 
    & MAE & 0.1214 & 0.1070 & \best{0.1031} \\
    & MSE & 0.0298 & 0.0248 & \best{0.0230} \\
    & RMSE & 0.1726 & 0.1575 & \best{0.1518} \\
  \bottomrule
  \end{tabular}
  \vspace{-0.15in}
\end{table}
\subsection{Experimental Setup}

\paragraph{Benchmark.}
We evaluate \method{} on four real-world datasets spanning financial and general domains. \textbf{Financial time-series datasets} include FNSPID \citep{dong2024fnspid} and Bitcoin, both exhibiting significant event-driven non-stationarity. \textbf{General time-series datasets} include Electricity and Environment, used to verify cross-domain generalization. For all datasets, we strictly follow the standard evaluation protocols and official data splits from the original literature to ensure fair comparison and reproducibility. Detailed dataset statistics are provided in Appendix~\ref{sec:detailed-dataset-statistics}.

\paragraph{Baselines.}
We compare \method{} against two categories of baselines: \cI~\textbf{Unimodal time-series models}, including TimeMixer \citep{wang2024timexer}, TSMixer \citep{chen2023tsmixer}, Nonstationary Transformer \citep{liu2022non}, TimesNet \citep{wu2023timesnet}, FEDformer \citep{zhou2022fedformer}, Pyraformer \citep{liu2022pyraformer}, Reformer \citep{kitaev2020reformer}, and PatchTST \citep{nie2022time}, covering mainstream architectures such as MLP and Transformer variants, all implemented using Time-Series-Library\footnote{\url{https://github.com/thuml/Time-Series-Library}}; \cII~\textbf{Multimodal fusion models}, including TimeLLM \citep{jin2024time}, ChatTime \citep{wang2025chattime}, and NewsForecasting \citep{wang2024news}, all using official implementations. Hyperparameters for all baselines are determined via grid search on the validation set to ensure fair comparison. Detailed baseline configurations are provided in Appendix~\ref{sec:detailed-baseline-configurations}.

\paragraph{Implementation Details.}
We implement all experiments in PyTorch \citep{paszke2019pytorch} on 8 NVIDIA A100 80GB GPUs, using AdamW \citep{loshchilov2017decoupled} with learning rate $\in{1\mathrm{e}{-4},5\mathrm{e}{-4},1\mathrm{e}{-3}}$ and early stopping (patience=10) on validation loss.

\subsection{Main Results}
Table~\ref{tab:comparison} presents the performance comparison between \method{} and all baselines across four datasets. 
\cI~\textbf{Significant improvements on financial datasets.}
On financial datasets exhibiting pronounced non-stationarity (FNSPID and Bitcoin), \method{} achieves substantial performance gains. Specifically, on FNSPID, \method{} improves over the strongest baseline TimesNet by 3.3\%, 20.0\%, and 9.9\% across three metrics. On Bitcoin, \method{} outperforms the strongest baseline NewsForecasting by 18.2\%, 29.1\%, and 15.8\% in terms of MAE, MSE, and RMSE, respectively. These results demonstrate the superiority of the temporal evolution semantic space in capturing event-driven non-stationary dynamics.
\cII~\textbf{Stable performance on general datasets.}
On general datasets, \method{} likewise demonstrates competitive performance. On Electricity, \method{} achieves the best performance across all metrics, improving over the strongest baseline by 0.4\%--5.0\%. On Environment, \method{} attains runner-up performance with less than 1\% gap from the strongest baseline Nonstationary Transformer. These results indicate that \method{} maintains stable predictive capability even in scenarios with relatively mild non-stationarity.

\paragraph{Effectiveness in Non-Stationary Scenarios.}
To further validate \method{}'s performance under different non-stationarity patterns, we extract three subsets from the test set, each representing a non-stationary scenario: \cI~\textbf{Shape transition}, \cII~\textbf{Volatility change}, and \cIII~\textbf{Mean shift}. Figure~\ref{fig:nonstationarity} presents the MSE performance across different methods. 
\method{} achieves consistent improvements across all three non-stationary scenarios, with MSE reductions of 21--52\% over multimodal baselines and 21--45\% over unimodal baselines on FNSPID and Bitcoin.

\subsection{Model Analysis}

\paragraph{Ablation Study.}
We ablate TESS and the gating mechanism on three datasets. Since gating relies on primitives extracted by TESS, removing TESS necessarily entails removing gating as well. Table~\ref{tab:ablation} presents the results. Removing TESS leads to substantial degradation, with MSE increasing by 46.2\%, 29.4\%, and 22.8\% respectively, confirming its effectiveness. In contrast, removing gating incurs more modest increases of 3.7\%, 2.6\%, and 7.5\%, indicating that confidence-aware gating effectively filters extraction errors. These results underscore the necessity of both components: TESS drives performance gains, while gating enhances robustness against LLM extraction errors.

\paragraph{Case Study}
Figure~\ref{fig:case_study} illustrates four representative prediction cases. On Bitcoin, \method{} successfully captures mean shift and trend patterns (e.g., \textsc{mild-rise}, \textsc{trough}); on Electricity, it identifies shape transitions with appropriate temporal localization (e.g., \textsc{mid-fade}, \textsc{late}). In contrast, direct text fusion fails to capture these patterns, leading to substantial prediction deviations. These results confirm that the semantic space effectively bridges implicit textual signals and quantifiable forecasting gains.

\begin{figure}[!tbp]
  \centering
  \includegraphics[width=\columnwidth]{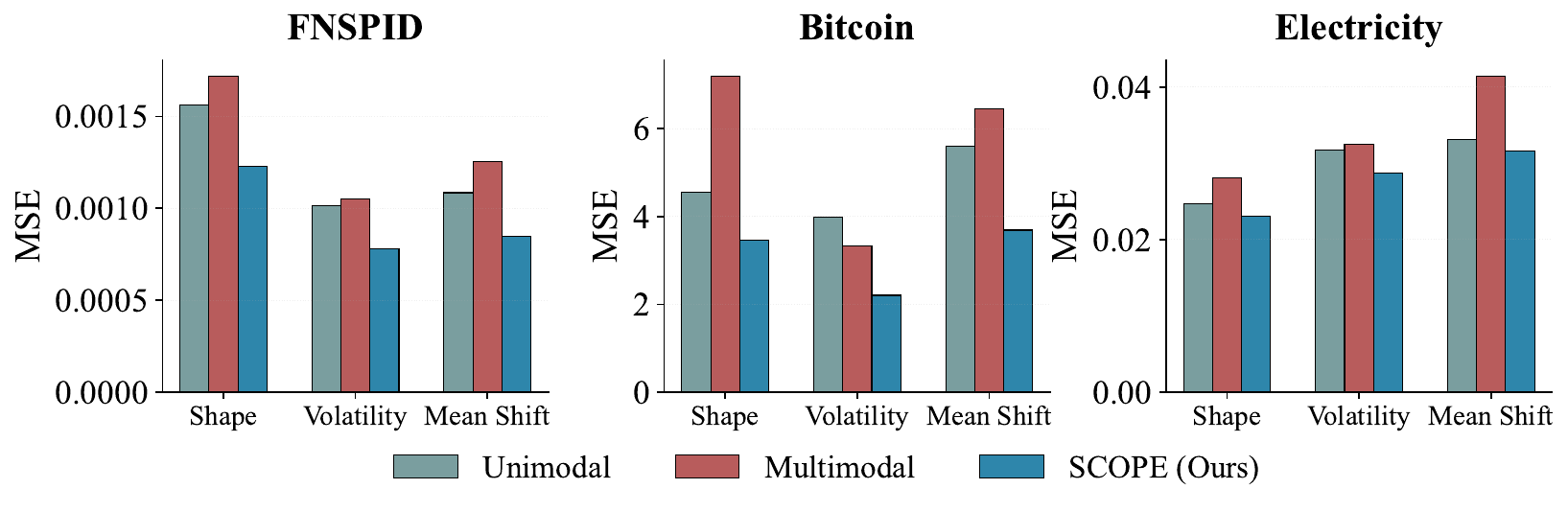}
  \vspace{-0.05in}
  \caption{Performance comparison on three types of non-stationary scenarios (Shape, Volatility, Mean Shift). \method{} (blue) consistently outperforms both unimodal (teal) and multimodal (red) baselines.}
  \label{fig:nonstationarity}
\end{figure}

\begin{figure}[!tbp]
  \centering
  \includegraphics[width=\columnwidth]{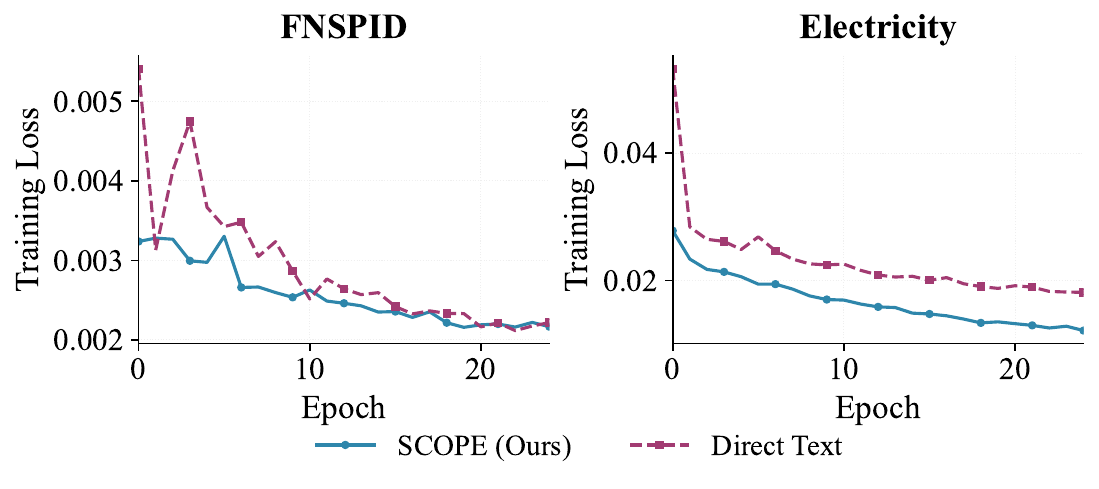}
  \vspace{-0.05in}
  \caption{Training loss curves comparison. \method{} (blue) converges faster with smoother dynamics than direct text fusion (purple).}
  \label{fig:loss_curves}
\end{figure}

\begin{figure}[!tbp]
  \centering
  \includegraphics[width=\columnwidth]{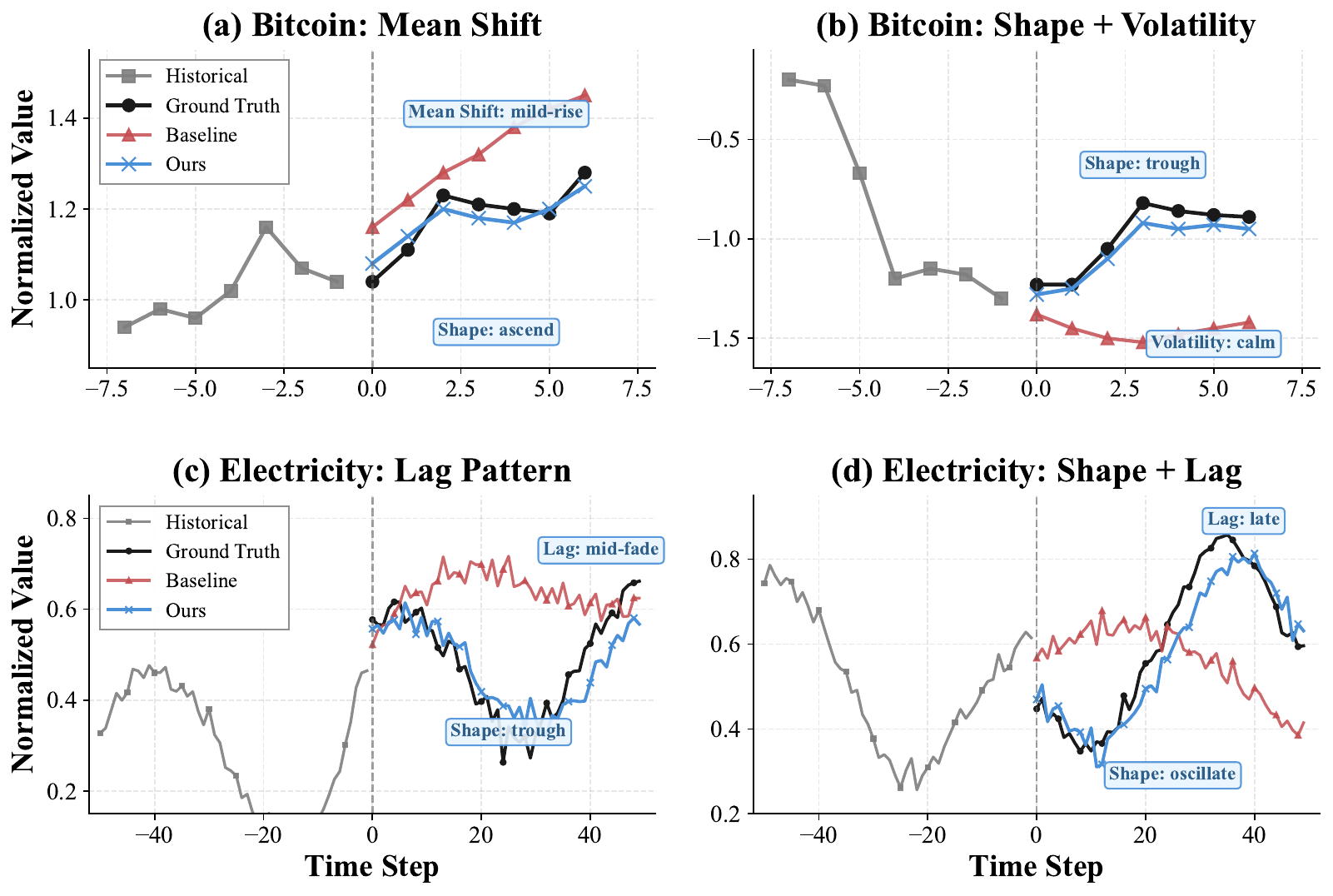}
  \vspace{-0.05in}
  \caption{Case study on Bitcoin and Electricity datasets. \method{} (blue) accurately captures temporal evolution primitives and aligns closely with ground truth (black).}
  \label{fig:case_study}
\end{figure}
\begin{figure}[!tbp]
  \centering
  \includegraphics[width=\columnwidth]{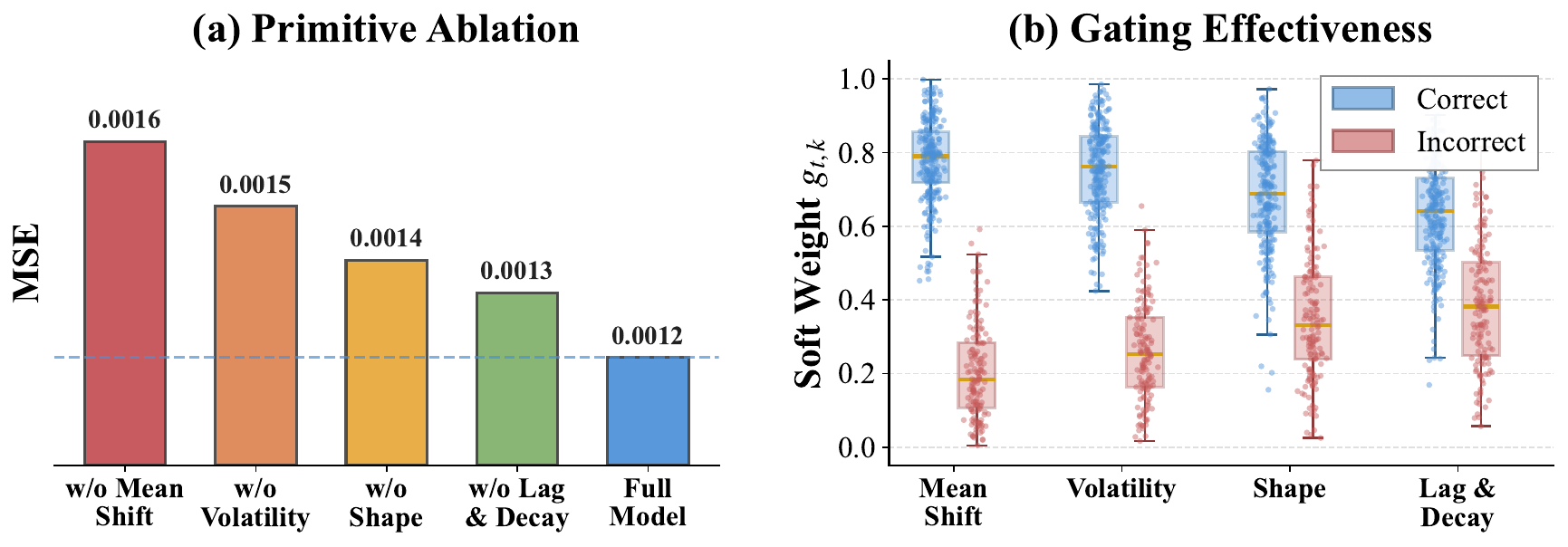}
  \vspace{-0.05in}
  \caption{Analysis of temporal evolution semantic space. \textbf{(a)} Primitive ablation on FNSPID: removing each primitive leads to MSE increase, with Mean Shift having the largest impact. \textbf{(b)} Gating effectiveness: correctly extracted primitives (blue) receive significantly higher soft weights than incorrect ones (red).}
  \label{fig:tess_analysis}
\end{figure}

\paragraph{Effectiveness of Temporal Evolution Semantic Space.}
 First, Figure~\ref{fig:loss_curves} shows that \method{} exhibits faster and smoother convergence than direct text fusion on both FNSPID and Electricity, indicating that the semantic space alleviates cross-modal optimization challenges. Second, Figure~\ref{fig:tess_analysis}(a) presents primitive ablation on FNSPID: removing Mean Shift alone increases MSE by 33\%, underscoring its role in capturing non-stationarity; removing other primitives also incurs performance loss. Best performance is achieved when all four primitives are jointly used, confirming their complementarity.

\paragraph{Effectiveness of Gating Mechanism.}
% we analyze the relationship between gating soft weights and primitive extraction correctness. 
Figure~\ref{fig:tess_analysis}(b) presents the soft weight distributions across the four primitives on the FNSPID dataset. The results demonstrate that gating weights are highly correlated with primitive correctness: correctly extracted samples exhibit median weights predominantly in the range of 0.65--0.78, whereas incorrectly extracted samples show median weights of only 0.21--0.40. This indicates that the gating network successfully learns to map the uncertainty signals from LLM outputs to reliability estimates, suppressing the interference of erroneous primitives on predictions through soft weighting.

\section{Conclusion}

We propose \method{}, a two-stage framework that bridges text-numerical modality gaps through a Temporal Evolution Semantic Space. A confidence-aware gating mechanism further ensures robustness against LLM extraction errors. Extensive experiments across financial and general domains demonstrate that \method{} achieves up to 29\% MSE reduction over state-of-the-art methods.

\section*{Impact Statement}

This paper presents work whose goal is to advance the field of time series analysis and multi model forecasting. There are many potential societal consequences of our work, none
which we feel must be specifically highlighted here.

\bibliography{example_paper}
\bibliographystyle{plainnat}

%%%%%%%%%%%%%%%%%%%%%%%%%%%%%%%%%%%%%%%%%%%%%%%%%%%%%%%%%%%%%%%%%%%%%%%%%%%%%%%
% APPENDIX
%%%%%%%%%%%%%%%%%%%%%%%%%%%%%%%%%%%%%%%%%%%%%%%%%%%%%%%%%%%%%%%%%%%%%%%%%%%%%%%
\newpage
\appendix
\onecolumn
\section{Theorems \& Proofs}
\label{app:theory}

Throughout this section, we write $(X,S,Y):=(\mathbf{X}_{\text{text}},\mathbf{X}_{\text{time}},\hat{\mathbf{Y}}_t)$
and $V:=\phi(X)$.

Let $\mathcal S=\{(X_i,Y_i)\}_{i=1}^n \sim p(x,y)^{\otimes n}$.
Given an encoder $p(z\mid x)$, draw $Z_i\sim p(\cdot\mid X_i)$ independently and
let $Z=(Z_1,\ldots,Z_n)$. Define
\[
L_{\textsc{pop}}(Z):=\mathbb{E}[\ell(Z,Y)],
\qquad
L_{\textsc{emp}}(Z):=\frac{1}{n}\sum_{i=1}^n \ell(Z_i,Y_i).
\]

\begin{assumption}
\label{assump:semantic-suff}
$Y \perp X \mid (V,S)$ (i.e., $p(y\mid x,s)=p(y\mid v,s)$ for $v=\phi(x)$).
\end{assumption}

\begin{assumption}
\label{ass:subgaussian}
$\ell(Z,y)$ is $\sigma$-sub-Gaussian under $Z\sim p(\cdot\mid x)$ for all $(x,y)$.
\end{assumption}

\begin{assumption}
\label{ass:independent-samples}
Conditioned on $\mathcal S$, $Z_i \perp Z_j$ for $i\neq j$; hence
$I(Z;\mathcal S)=\sum_{i=1}^n I(Z_i;X_i)=n\,I(Z;X)$.
\end{assumption}

\begin{theorem*}[Restatement of Theorem~\ref{thm:main}]
Assume semantic sufficiency (Assump.~\ref{assump:semantic-suff}) in our forecasting setting:
$\hat{\mathbf{Y}}_t \perp\!\!\!\perp \mathbf{X}_{\text{text}} \mid (\mathbf{P}_t,\mathbf{X}_{\text{time}})$,
where $\mathbf{P}_t$ denotes the distilled primitives.
Then $\forall$ encoders $f:\mathbf{X}_{\text{text}}\!\mapsto Z$, $\exists$ an encoder $\tilde f:\mathbf{P}_t\!\mapsto \tilde Z$ s.t.
\[
I(\tilde Z;\hat{\mathbf{Y}}_t\mid \mathbf{X}_{\text{time}})
= I(Z;\hat{\mathbf{Y}}_t\mid \mathbf{X}_{\text{time}}),
\]
\[
I(\tilde Z;\mathbf{X}_{\text{text}})
\le I(Z;\mathbf{X}_{\text{text}}).
\]

Moreover, under standard sub-Gaussian loss and i.i.d.\ sampling Assumption ~\ref{ass:subgaussian}, and ~\ref{ass:independent-samples},
\[
\textsc{Gen}(\tilde Z)\le \textsc{Gen}(Z),
\]
where $\textsc{Gen}(\cdot)$ denotes the expected population--empirical loss gap of a predictor trained on the representation.
\end{theorem*}

\begin{proof}
% We denote by $p$ both the data distribution $p(x,s,y)$ and the original encoder
% distribution $p(z \mid x)$, with the understanding that the factorization
% of the joint over $(X,S,Y,Z)$ is
% \begin{equation}
% p(x,s,y,z) = p(z \mid x)\, p(y \mid x,s)\, p(x,s),
% \end{equation}
% which follows from the chain rule and the assumption that $Z$ is generated
% from $X$ only, i.e., $Z \perp (Y,S)\mid X$.
% The conditional mutual information and mutual information are
% \begin{align}
% I_p(Z;Y \mid S)
% &:= \sum_{s} p(s) \sum_{z,y} p(z,y \mid s)
% \log \frac{p(z,y \mid s)}{p(z \mid s)\, p(y \mid s)},
% \label{eq:def-Izys}
% \\
% I_p(Z;X)
% &:= \sum_{x,z} p(x,z)\log \frac{p(x,z)}{p(x)\,p(z)}
% = \sum_{x} p(x)\sum_z p(z \mid x)\log \frac{p(z \mid x)}{p(z)}.
% \label{eq:def-Izx}
% \end{align}
% Our goal is to construct a new encoder $\tilde{p}(z \mid x)$ and show that
% \eqref{eq:thm-eq}--\eqref{eq:thm-ineq} hold.

% Let $V=\phi(X)$ and define
% \begin{equation}
% p(v,s) := \sum_{x : \phi(x)=v} p(x,s),
% \label{eq:def-pvs}
% \end{equation}
% and, for all $x$ such that $\phi(x)=v$,
% \begin{equation}
% \alpha_{x \mid v,s}
% := \frac{p(x,s)}{p(v,s)}.
% \label{eq:def-alpha}
% \end{equation}
% Then $\sum_{x:\phi(x)=v}\alpha_{x\mid v,s}=1$ and hence
% \begin{align}
% p(z \mid v,s)
% &:= \sum_{x : \phi(x) = v} \alpha_{x \mid v,s}\, p(z \mid x),
% \label{eq:def-pzvs}
% \\
% \sum_z p(z \mid v,s)
% &= \sum_{x : \phi(x) = v} \alpha_{x \mid v,s} \sum_z p(z \mid x)
% = \sum_{x : \phi(x) = v} \alpha_{x \mid v,s}
% = 1,
% \nonumber
% \end{align}
% so $p(z\mid v,s)$ is a valid conditional distribution. 

We use $p$ to denote both the data law $p(x,s,y)$ and the original encoder $p(z\mid x)$.
The induced joint on $(X,S,Y,Z)$ factorizes as
\[
p(x,s,y,z)=p(z\mid x)\,p(y\mid x,s)\,p(x,s),
\]
since $Z\perp (Y,S)\mid X$. Define
\begin{align}
I_p(Z;Y\mid S)
&:=\sum_s p(s)\sum_{z,y} p(z,y\mid s)\log\frac{p(z,y\mid s)}{p(z\mid s)p(y\mid s)},
\label{eq:def-Izys}\\
I_p(Z;X)
&:=\sum_x p(x)\sum_z p(z\mid x)\log\frac{p(z\mid x)}{p(z)}.
\label{eq:def-Izx}
\end{align}

Let $V=\phi(X)$. Define
\[
p(v,s):=\sum_{x:\phi(x)=v}p(x,s),\qquad
\alpha_{x\mid v,s}:=\frac{p(x,s)}{p(v,s)}\ \ ( \phi(x)=v),
\]
so $\sum_{x:\phi(x)=v}\alpha_{x\mid v,s}=1$. Then set
\begin{equation}
p(z\mid v,s):=\sum_{x:\phi(x)=v}\alpha_{x\mid v,s}\,p(z\mid x),
\label{eq:def-pzvs}
\end{equation}
which is a valid conditional distribution. Define the new encoder
\begin{align}
\tilde{p}(z \mid x)
&:= p(z \mid v,s),
\qquad v=\phi(x),
\label{eq:def-tilde-encoder}
\end{align}
which depends on $x$ only through $v=\phi(x)$ (and, in the joint construction, $s$).

Starting from \eqref{eq:def-Izys}, it suffices to prove that the induced joint
over $(Z,Y,S)$ is unchanged:
\begin{equation}
p_{\tilde p}(z,y,s)=p(z,y,s)\ \ \forall z,y,s
\ \Longrightarrow\
I_{\tilde p}(Z;Y\mid S)=I_p(Z;Y\mid S).
\label{eq:suffice-pzys}
\end{equation}
Under $p(z\mid x)$, using semantic sufficiency $p(y\mid x,s)=p(y\mid v,s)$ with $v=\phi(x)$,
\begin{align}
p(z,y,s)
&=\sum_x p(z\mid x)p(y\mid x,s)p(x,s)
=\sum_v p(y\mid v,s)p(v,s)\sum_{x:\phi(x)=v}\alpha_{x\mid v,s}p(z\mid x) \notag\\
&=\sum_v p(y\mid v,s)p(v,s)p(z\mid v,s).
\label{eq:pzys}
\end{align}
Under $\tilde p(z\mid x)=p(z\mid v,s)$,
\begin{align}
p_{\tilde p}(z,y,s)
&=\sum_x \tilde p(z\mid x)p(y\mid x,s)p(x,s)
=\sum_v p(z\mid v,s)p(y\mid v,s)\sum_{x:\phi(x)=v}p(x,s) \notag\\
&=\sum_v p(z\mid v,s)p(y\mid v,s)p(v,s)
= p(z,y,s),
\label{eq:pzys-tilde}
\end{align}
so $I_{\tilde p}(Z;Y\mid S)=I_p(Z;Y\mid S)$.

Next, recall \eqref{eq:def-Izx}. For $p$, define $p(z):=\sum_x p(z\mid x)p(x)$.
Grouping by $(v,s)$ yields
\begin{equation}
I_p(Z;X)
=\sum_{v,s} p(v,s)\sum_{x:\phi(x)=v}\alpha_{x\mid v,s}
\sum_z p(z\mid x)\log\frac{p(z\mid x)}{p(z)}.
\label{eq:IpZX-vs}
\end{equation}
For $\tilde p$, the $Z$-marginal matches:
\begin{equation}
p_{\tilde p}(z)
=\sum_x \tilde p(z\mid x)p(x)
=\sum_{v,s} p(v,s)p(z\mid v,s)
=\sum_{v,s} p(v,s)\sum_{x:\phi(x)=v}\alpha_{x\mid v,s}p(z\mid x)
= p(z),
\label{eq:pz-same}
\end{equation}
hence the denominator is unchanged. Therefore
\begin{equation}
I_{\tilde p}(Z;X)
=\sum_{v,s} p(v,s)\sum_z p(z\mid v,s)\log\frac{p(z\mid v,s)}{p(z)}.
\label{eq:Izx-tilde}
\end{equation}

Define $f(r):=\sum_z r(z)\log\frac{r(z)}{p(z)}=D_{\mathrm{KL}}(r\Vert p(z))$.
By convexity of $D_{\mathrm{KL}}(\cdot\Vert p(z))$ in its first argument, for each fixed $(v,s)$,
\begin{align}
\sum_z p(z\mid v,s)\log\frac{p(z\mid v,s)}{p(z)}
&= f\!\left(\sum_{x:\phi(x)=v}\alpha_{x\mid v,s}\,p(\cdot\mid x)\right)
\le \sum_{x:\phi(x)=v}\alpha_{x\mid v,s}\,f(p(\cdot\mid x))
\notag\\
&=\sum_{x:\phi(x)=v}\alpha_{x\mid v,s}\sum_z p(z\mid x)\log\frac{p(z\mid x)}{p(z)}.
\label{eq:key-inequality}
\end{align}
Multiplying \eqref{eq:key-inequality} by $p(v,s)$ and summing over $(v,s)$ yields
\begin{align}
I_{\tilde p}(Z;X)
&=\sum_{v,s} p(v,s)\sum_z p(z\mid v,s)\log\frac{p(z\mid v,s)}{p(z)}
\le \sum_{v,s} p(v,s)\sum_{x:\phi(x)=v}\alpha_{x\mid v,s}\sum_z p(z\mid x)\log\frac{p(z\mid x)}{p(z)}
\notag\\
&=\sum_{v,s}\sum_{x:\phi(x)=v} p(x,s)\sum_z p(z\mid x)\log\frac{p(z\mid x)}{p(z)}
=\sum_x p(x)\sum_z p(z\mid x)\log\frac{p(z\mid x)}{p(z)}
=I_p(Z;X),
\end{align}

Follow the information-theoretic generalization framework of \citet{xu2017information}.
By Assumption~\ref{ass:subgaussian}, Theorem~1 of \citet{xu2017information} gives
\begin{equation}
\left|
\mathbb{E}\!\left[L_{\textsc{pop}}(Z)-L_{\textsc{emp}}(Z)\right]
\right|
\le
\sqrt{\frac{2\sigma^2}{n}\,I(Z;S)}.
\label{eq:gen-IS}
\end{equation}
By Assumption~\ref{ass:independent-samples},
\begin{equation}
I(Z;S)=\sum_{i=1}^n I(Z_i;X_i)=n\,I(Z;X).
\label{eq:IZS-factor}
\end{equation}
Combining \eqref{eq:gen-IS}--\eqref{eq:IZS-factor} yields
\begin{equation}
\textsc{Gen}(Z)
:=
\left|
\mathbb{E}\!\left[L_{\textsc{pop}}(Z)-L_{\textsc{emp}}(Z)\right]
\right|
\le
\sqrt{\frac{2\sigma^2}{n}\,I_p(Z;X)}.
\label{eq:z-gen-bound}
\end{equation}
Applying the same argument to $\tilde Z$ gives
\begin{equation}
\textsc{Gen}(\tilde Z)
\le
\sqrt{\frac{2\sigma^2}{n}\,I_{\tilde p}(Z;X)}.
\label{eq:z-tilde-gen-bound}
\end{equation}
Finally, since $I_{\tilde p}(Z;X)\le I_p(Z;X)$, we obtain
\[
\textsc{Gen}(\tilde Z)\le \textsc{Gen}(Z).
\]
\end{proof}

\begin{assumption}
\label{assump:lipschitz_gated_primitives}
Fix $t$ and let $\hat Y_t=F(E_{\mathrm{time},t},\tilde h_1,\ldots,\tilde h_K)$ with
$\tilde h_k=g_{t,k}h_k$ and $g_{t,k}\in[0,1]$.
Assume $F$ is coordinate-wise $L_k$-Lipschitz in $\tilde h_k$ (w.r.t.\ $\|\cdot\|$):
\[
\big|F(\ldots,a,\ldots)-F(\ldots,b,\ldots)\big|\le L_k\|a-b\|,
\quad \forall\,a,b\in\mathbb R^d,\ \forall k\in[K].
\]
\end{assumption}

\begin{theorem}\label{thm:soft_gating_stability}
Under Assump.~\ref{assump:lipschitz_gated_primitives}, define
$\Delta_k:=\|h_k^{\err}-h_k^{\true}\|$ and $\tilde h_k^{\true}:=g_{t,k}h_k^{\true}$,
$\tilde h_k^{\err}:=g_{t,k}h_k^{\err}$.
Let $E:=E_{\timev,t}$ and $\hat Y_t^{\true}:=F(E,\tilde h_1^{\true},\ldots,\tilde h_K^{\true})$,
$\hat Y_t^{\err}:=F(E,\tilde h_1^{\err},\ldots,\tilde h_K^{\err})$.
Then
\[
(\hat Y_t^{\err}-\hat Y_t^{\true})^2
\le K\sum_{k=1}^K L_k^2\,g_{t,k}^2\,\Delta_k^2 .
\]
\end{theorem}

\begin{proof}
For $k=0,\ldots,K$, set
$\bz^{(k)}:=(g_{t,1}h_1^{\err},\ldots,g_{t,k}h_k^{\err},g_{t,k+1}h_{k+1}^{\true},\ldots,g_{t,K}h_K^{\true})$,
so $\bz^{(0)}=(\tilde h_1^{\true},\ldots,\tilde h_K^{\true})$ and $\bz^{(K)}=(\tilde h_1^{\err},\ldots,\tilde h_K^{\err})$.
By telescoping and Assump.~\ref{assump:lipschitz_gated_primitives},
\[
|\hat Y_t^{\err}-\hat Y_t^{\true}|
\le\sum_{k=1}^K |F(E,\bz^{(k)})-F(E,\bz^{(k-1)})|
\le\sum_{k=1}^K L_k\,\|g_{t,k}(h_k^{\err}-h_k^{\true})\|
=\sum_{k=1}^K L_k g_{t,k}\Delta_k .
\]
Apply $(\sum_{k=1}^K a_k)^2\le K\sum_{k=1}^K a_k^2$ with $a_k=L_k g_{t,k}\Delta_k$.
\end{proof}

% \begin{theorem}

% \label{thm:gen}
% Let $V=\phi(X)\in\mathcal V$ denote the TESS primitives, where
% $\mathcal V=\mathcal V_1\times\cdots\times\mathcal V_K$ with $|\mathcal V_k|=M_k$ and
% $|\mathcal V|=M:=\prod_{k=1}^K M_k$.
% Consider predictors $\hat Y=g(V)$ with $g:\mathcal V\to[-1,1]$, squared loss
% $\ell(\hat y,y)=(\hat y-y)^2$, and risks
% \[
% R(g)=\mathbb E\,\ell(g(V),Y),\qquad
% \widehat R_n(g)=\frac1n\sum_{i=1}^n \ell(g(V_i),Y_i).
% \]
% Then there exists a universal constant $C>0$ such that for any $\delta\in(0,1)$, with probability at least
% $1-\delta$ over $n$ i.i.d.\ samples,
% \[
% \sup_{g:\mathcal V\to[-1,1]}
% \big|R(g)-\widehat R_n(g)\big|
% \le
% C\!\left(\sqrt{\frac{M}{n}}+\sqrt{\frac{\log(1/\delta)}{n}}\right).
% \]
% Consequently, $n\gtrsim M/\varepsilon^2$ suffices to ensure
% $\sup_g |R(g)-\widehat R_n(g)|\le\varepsilon$.
% \end{theorem}

\begin{theorem}\label{thm:gen}
Let $V=\phi(X)\in\cV$ be the TESS primitives, where
$\cV=\cV_1\times\cdots\times\cV_K$, $|\cV_k|=M_k$, and $|\cV|=:M=\prod_{k=1}^K M_k$.
For $g:\cV\to[-1,1]$, define $\hat Y=g(V)$, squared loss $\ell(\hat y,y)=(\hat y-y)^2$, and
\[
R(g):=\E\,\ell(g(V),Y),\qquad
\hat R_n(g):=\frac1n\sum_{i=1}^n \ell\!\big(g(V_i),Y_i\big).
\]
Then $\exists$ a universal $C>0$ s.t.\ $\forall\delta\in(0,1)$, w.p.\ $\ge 1-\delta$ over $n$ i.i.d.\ samples,
\[
\sup_{g:\cV\to[-1,1]}\big|R(g)-\hat R_n(g)\big|
\;\le\;
C\!\left(\sqrt{\frac{M}{n}}+\sqrt{\frac{\log(1/\delta)}{n}}\right).
\]
In particular, $n\gtrsim M/\e^2 \Rightarrow \sup_g |R(g)-\hat R_n(g)|\le \e$.
\end{theorem}

\begin{proof}
Define $\mathcal G:=\{g:\mathcal V\to[-1,1]\}$ and its empirical Rademacher complexity
\[
\widehat{\mathfrak R}_n(\mathcal G)
=
\mathbb E_\sigma\Big[\sup_{g\in\mathcal G}\frac1n\sum_{i=1}^n\sigma_i g(V_i)\Big],
\qquad \sigma_i\overset{iid}\sim\{\pm1\}.
\]

Group samples by primitive value and let
$N_v:=|\{i:V_i=v\}|$ so that $\sum_{v\in\mathcal V}N_v=n$. Then
\[
\sup_{g\in\mathcal G}\sum_{i=1}^n\sigma_i g(V_i)
=
\sup_{g}\sum_{v\in\mathcal V} g(v)\!\sum_{i:V_i=v}\sigma_i
=
\sum_{v\in\mathcal V}\Big|\sum_{i:V_i=v}\sigma_i\Big|.
\]
Taking expectation over $\sigma$ and using
$\mathbb E_\sigma|\sum_{j=1}^m\sigma_j|\le\sqrt{m}$ yields
\[
\widehat{\mathfrak R}_n(\mathcal G)
\le
\frac1n\sum_{v\in\mathcal V}\sqrt{N_v}.
\]
By Cauchy--Schwarz,
\[
\sum_{v\in\mathcal V}\sqrt{N_v}
\le
\sqrt{\Big(\sum_v 1\Big)\Big(\sum_v N_v\Big)}
=
\sqrt{Mn},
\]
hence $\widehat{\mathfrak R}_n(\mathcal G)\le\sqrt{M/n}$.

Since $\ell(\cdot,y)$ is $4$-Lipschitz on $[-1,1]$, the contraction inequality gives
\[
\widehat{\mathfrak R}_n(\ell\circ\mathcal G)\le 4\,\widehat{\mathfrak R}_n(\mathcal G)
\le 4\sqrt{\frac{M}{n}}.
\]
Applying the standard Rademacher generalization bound (e.g.\ Bartlett \& Mendelson, 2002),
with probability at least $1-\delta$,
\[
\sup_{g\in\mathcal G}\big(R(g)-\widehat R_n(g)\big)
\le
2\,\widehat{\mathfrak R}_n(\ell\circ\mathcal G)
+ c\sqrt{\frac{\log(1/\delta)}{n}}
\le
C\!\left(\sqrt{\frac{M}{n}}+\sqrt{\frac{\log(1/\delta)}{n}}\right),
\]
for universal constants $c,C>0$. The same bound holds for the absolute deviation.
\end{proof}

\begin{remark}

Theorem~\ref{thm:gen} gives a complexity term $\sqrt{M/n}$ with $M=\prod_k M_k$, compared to the token-level $\sqrt{\log|\mathcal A_T|/n}$; in particular, for $|\mathcal A_T|=2^T$ this improves $\sqrt{T/n}$ to $\sqrt{(\prod_k M_k)/n}$ when $\prod_k M_k\ll T$.
\end{remark}

% \begin{theorem}[TESS reduces effective complexity]
% \label{thm:tess_simpler}
% Let token-level selection be an $m_{\mathrm{tok}}$-class task. TESS yields $K$ primitive tasks
% with $m_{\max}:=\max_k|\mathcal V_k|$. For a class $\mathcal H$ of complexity $d$, there exists
% $C>0$ such that w.p.\ $\ge 1-\delta$,
% \[
% R_{\mathrm{tok}}(h)\le \widehat R_{\mathrm{tok}}(h)+C\sqrt{\frac{d\log m_{\mathrm{tok}}+\log(1/\delta)}{n}},
% \qquad
% \max_{k\le K}R_k(h)\le \max_{k\le K}\widehat R_k(h)+C\sqrt{\frac{d\log m_{\max}+\log(K/\delta)}{n}},
% \]
% for all $h\in\mathcal H$. Hence achieving excess $\le\varepsilon$ needs
% $n\gtrsim(d\log m_{\mathrm{tok}}+\log(1/\delta))/\varepsilon^2$ vs.
% $n\gtrsim(d\log m_{\max}+\log(K/\delta))/\varepsilon^2$.
% \end{theorem}

% \begin{proof}
% Apply a standard $m$-class bound to the token task ($m=m_{\mathrm{tok}}$).
% For primitives, apply it to each $k$ with confidence $\delta/K$ and union bound; use
% $\log m_k\le\log m_{\max}$. Solve the excess term $\le\varepsilon$ for $n$.
% \end{proof}

\section{ Detailed dataset statistics }
\label{sec:detailed-dataset-statistics}

\begin{table}[H]
  \centering
  \caption{ The Electricity dataset represents the half-hourly electricity demand in a state. FNSPID provides daily stock price data integrated with time-aligned financial news. Environment dataset contains daily Air Quality Index (AQI) measurements and Bitcoin denotes the daily Bitcoin price.}
  \label{tab:appendix-b}
  \small
  \begin{tabular}{|c|>{\centering\arraybackslash}p{2.8cm}|>{\centering\arraybackslash}p{2.8cm}|>{\centering\arraybackslash}p{2.8cm}|>{\centering\arraybackslash}p{2.8cm}|}
  \hline
  \textbf{Datasets} & \textbf{Electicity} & \textbf{Bitcoin} & \textbf{FNSPID} & \textbf{Environment} \\
  \hline
  \textbf{Time Horizon} & 2019.01-2021.12 & 2019.01-2021.06 & Not specified & 1982.01-2012.05 \\
  \hline
  \textbf{Variates} & 19 & 18 & 7 & 4 \\
  \hline
  \textbf{Timesteps} & 52,560 & 858 & 49140 & 11,095 \\
  \hline
  \textbf{Granularity} & 30 minutes & 1 day & 5 days & 7 days \\
  \hline
  \textbf{Input length} & 48 & 7 & 5 & 7 \\
  \hline
  \textbf{Prediction length} & 48 & 7 & 5 & 7 \\
  \hline
  \end{tabular}
\end{table}
We conduct experiments on four real-world multimodal datasets from diverse domains, namely Bitcoin, FNSPID, Electricity, and Environment. Each dataset consists of a target time series accompanied by temporally aligned news text describing relevant external events and contextual information.

The news corpus corresponding to each time series is collected from multiple financial news archives and publicly accessible news repositories, including established sources such as the Reuters News Archive. In addition, we gather supplementary news articles from Google News through third-party aggregation tools, which aggregate and index articles from a wide range of media outlets.

All collected news articles are timestamped and systematically aligned with the corresponding time-series observations based on their publication time, ensuring temporal consistency between the textual and numerical modalities. This alignment enables the model to effectively leverage contemporaneous textual information for time-series forecasting.

\section{Detailed baseline configurations}
\label{sec:detailed-baseline-configurations}
To ensure fair comparison, we conduct systematic hyperparameter search on the validation set for all baseline methods. For architectural parameters, we maintain consistency across models where applicable: embedding dimension is uniformly set to 512, dropout rate to 0.1, and activation function primarily uses GELU. Most encoder-decoder architectures employ 2 encoder layers and 1 decoder layer to balance model capacity and avoid overfitting on smaller datasets.

Learning rate is the primary hyperparameter subject to grid search. For traditional time-series models, we search over three standard values: 0.0005, 0.0001, and 0.00005. Models exhibiting stable performance in preliminary experiments (e.g., TimesNet) directly adopt 0.0005 to reduce computational cost. For multimodal models with more complex training dynamics, the search range extends to [0.005, 0.001, 0.0001, 0.0005]. Additional architecture-specific searches include: decoder layers for FEDformer (1 or 2 layers), attention heads for PatchTST (ranging from 2 to 8), and factor parameters for Pyraformer (1 or 3). Dataset-specific adjustments are made when necessary.For instance, Pyraformer's window size is adjusted to [4, 4] for Bitcoin's short sequences, and TimeLLM's top\_k is reduced to 3 for smaller datasets like FNSPID.

\section{Prompts}
In this section, we list the prompt template we use throughout the paper. Other than ensuring that generations could be parsed properly, our prompt were not optimized towards any particular models.

\subsection{Template for Temporal Evolution Primitives Extraction}

\begin{promptbox}[Prompt Template]
\small
\begin{verbatim}
You are a professional {role from the dataset background}. 
Your task is to analyze the provided textual information and 
infer temporal evolution patterns that may impact future time series behavior.

Textual Input: {text_content}
Domain Context: {domain_context}

Instructions:
Based on the textual content provided below, analyze and classify 
the following four temporal evolution primitives:

1. Mean Shift - Infer the direction and magnitude of anticipated 
   level changes:
   <strong-rise | mild-rise | stable | mild-drop | strong-drop>

2. Volatility - Infer the anticipated changes in volatility regime:
   <surge | rise | stable | fall | calm>

3. Shape - Infer the dominant trend morphology pattern over the 
   forecast horizon:
   <ascend | descend | peak | trough | oscillate>

4. Lag and Decay - Infer the temporal localization and persistence 
   of the impact:
   <early-fade | early-persist | mid-fade | mid-persist | late | diffuse>

You MUST output in the following EXACT format with no extra text:

Mean Shift: <strong-rise | mild-rise | stable | mild-drop | strong-drop>
Volatility: <surge | rise | stable | fall | calm>
Shape: <ascend | descend | peak | trough | oscillate>
Lag: <early-fade | early-persist | mid-fade | mid-persist | late | diffuse>

Provide your analysis in the exact format specified above.
\end{verbatim}
\end{promptbox}

\end{document}